\documentclass{article}
\PassOptionsToPackage{authoryear,round}{natbib}

\usepackage[main,final]{neurips_2026}

\usepackage[utf8]{inputenc} 
\usepackage[T1]{fontenc}    
\usepackage{hyperref}       
\usepackage{url}            
\usepackage{booktabs}       
\usepackage{amsmath}
\usepackage{amsfonts}       
\usepackage{nicefrac}       
\usepackage{microtype}      
\usepackage{xcolor}         

\usepackage{listings}
\usepackage{tcolorbox}
\usepackage{makecell}
\usepackage{enumitem}
\usepackage{stfloats}
\usepackage{float}
\usepackage{algorithm}
\usepackage{algpseudocode}
\usepackage{thmtools}
\usepackage{thm-restate}

\newtheorem{assumption}{Assumption}
\newtheorem{definition}{Definition}
\newtheorem{proposition}{Proposition}
\newtheorem{lemma}{Lemma}
\newtheorem{remark}{Remark}

\usepackage{wrapfig}

\usepackage{pifont}
\usepackage{graphicx}
\usepackage{subcaption}

\definecolor{trendred}{RGB}{231, 75, 120}
\definecolor{trendgreen}{RGB}{30, 180, 120}

\usepackage[dvipsnames]{xcolor} 
\usepackage{listings}
\usepackage{multirow}
\usepackage{booktabs}
\definecolor{lightred}{RGB}{255,230,230}

\newcommand{\yes}{\textcolor{trendgreen}{\ding{51}}}
\newcommand{\no}{\textcolor{trendred}{\ding{55}}}

\usepackage[table]{xcolor}
\definecolor{scoreA}{HTML}{FBFDFF} 
\definecolor{scoreB}{HTML}{F1F7FD}
\definecolor{scoreC}{HTML}{E3F0FA}
\definecolor{scoreD}{HTML}{D2E6F5}
\definecolor{scoreE}{HTML}{BCD8EE}
\definecolor{scoreF}{HTML}{9EC5E5} 

\newcommand{\sA}[1]{\cellcolor{scoreA}#1}
\newcommand{\sB}[1]{\cellcolor{scoreB}#1}
\newcommand{\sC}[1]{\cellcolor{scoreC}#1}
\newcommand{\sD}[1]{\cellcolor{scoreD}#1}
\newcommand{\sE}[1]{\cellcolor{scoreE}#1}
\newcommand{\sF}[1]{\cellcolor{scoreF}#1}

\newcommand{\method}{\textsc{TPO}}

\title{TPO: Tri-level Distributionally Robust Learning for OOD Direct Preference Optimization}

\author{%
  Chengtao Jian$^{1}$, Kai Yang$^{1,2,3}$\thanks{Corresponding author.} \\
  $^{1}$School of Computer Science and Technology, Tongji University, Shanghai, China \\
  $^{2}$Shenzhen Loop Area Institute, Shenzhen, China \\
  $^{3}$Shanghai Research Institute for Intelligent Autonomous Systems, Shanghai, China \\
  \texttt{\{jct,kaiyang\}@tongji.edu.cn}
}

\begin{document}

\maketitle

\begin{abstract}
Direct Preference Optimization (DPO) minimizes empirical risk over observed preference comparisons, implicitly assuming that the training preference distribution matches the deployment preference distribution. In practice, this distributional assumption often fails, as preference distribution shifts can arise at three levels: group-level shift, prompt-level shift, and preference-level shift. Such multi-level distributional mismatch poses a critical challenge for safe LLM deployment, yet existing robust DPO methods typically either model uncertainty over the joint preference distribution or focus on individual sources of shift, leaving the hierarchical structure of OOD preference shifts underexplored. We propose \method{}, a tri-level distributionally robust framework for OOD preference alignment, which jointly models all three levels of distributional uncertainty through a hierarchically structured rectangular ambiguity set. By instantiating it with Wasserstein distance, we derive a tractable and interpretable objective that combines a preference-level margin penalty, prompt-level adversarial embedding perturbation, and worst-case group reweighting, each governed by an independent radius. We establish finite-sample convergence guarantees and empirically validate the effectiveness of \method{} across diverse OOD environments.
\end{abstract}

\section{Introduction}

Direct Preference Optimization (DPO)~\citep{RafailovSMMEF23_DPO} has emerged as a prominent method for aligning large language models (LLMs) with human preferences. Unlike RLHF~\citep{ChristianoLBMLA17,StiennonO0ZLVRA20,Ouyang0JAWMZASR22}, DPO directly optimizes a preference objective from pairwise comparison data, bypassing explicit reward modeling and improving optimization stability. However, standard DPO minimizes empirical risk over observed pairwise preference data, without explicitly accounting for distribution shifts between training and deployment. Such out-of-distribution (OOD) shifts can degrade preference alignment performance~\citep{wu2025iclr-robust_drdpo,xu2025neurips-robust_wdpo}.

In preference alignment, OOD preference shift exhibits a natural hierarchical structure across multiple levels. 
\textbf{First}, the distribution over groups may change between training and deployment, inducing a \emph{group-level shift}, which has been studied in group robust preference alignment~\citep{ramesh2024group_grpo,SagawaKHL20_groupdro}. 
\textbf{Second}, even within the same group, the prompt marginal distribution may vary across domains, time, or user interaction patterns, yielding a \emph{prompt-level shift}, which corresponds to covariate shift in the prompt distribution and has been observed in robust prompt optimization~\citep{li-etal-2023-robust,seegmiller2026measuring}. 
\textbf{Third}, the conditional preference distribution may also change due to noisy feedback or heterogeneous preferences within the same group, giving rise to a \emph{preference-level shift}~\citep{YaoCCYJYH25_gdpo,chowdhury24a_rdpo,PoddarWIGJ24}.

Recent robust DPO methods have begun to address distributional robustness, but they typically do not explicitly model the hierarchical structure of preference distribution shift. WDPO and KLDPO~\citep{xu2025neurips-robust_wdpo} construct Wasserstein and KL ambiguity sets over the entire preference distribution, while rDPO~\citep{chowdhury24a_rdpo} and Dr.DPO~\citep{wu2025iclr-robust_drdpo} focus on preference label noise and on pointwise and pairwise data noise, respectively.
As summarized in Table~\ref{tab:method-comparison}, existing methods either model overall uncertainty or focus on a single robustness level. Motivated by these observations, we propose \method{},
a tri-level distributionally robust framework for OOD
preference alignment, which jointly models all three levels through a hierarchically structured ambiguity set. Overall, our main contributions are as follows:
\begin{itemize}[leftmargin=*]
\item \textbf{Tri-level learning for OOD preference alignment.} 
We identify that preference distribution shift exhibits
a natural tri-level structure comprising group-level,
prompt-level, and preference-level shifts that existing
robust DPO methods do not jointly address. We propose 
\method{}, which captures all three levels within a hierarchical 
rectangular ambiguity set. This decomposition is relevant to safe LLM alignment, as real-world LLMs inevitably encounter distributional mismatch at multiple levels simultaneously, while single-level robustness leaves models vulnerable to the remaining sources of shift.

\item \textbf{Tractable Wasserstein robust preference objective.} 
By instantiating the hierarchical ambiguity set with Wasserstein distances, 
we derive a tractable and interpretable training objective with three 
complementary components: a preference-level margin penalty, 
prompt-level adversarial embedding perturbation, and worst-case group 
reweighting. Each component targets a distinct source of distributional uncertainty and is governed by its own radius, enabling flexible and modular control over robustness.

\item \textbf{Theoretical guarantees and empirical validation.} 
We establish finite-sample convergence guarantees, showing that the empirical minimizer converges to the population minimizer at a rate of $\widetilde{\mathcal{O}}(n_{\min}^{-1/4})$, where $n_{\min}=\min_{g\in[m]}n_g$ is the smallest group size.
Empirically, \method{} achieves superior or competitive performance over standard DPO and robust DPO baselines across multiple OOD shift settings.
\end{itemize}

\begin{table}[t]
\centering
\small
\vspace{-0.6em}
\caption{Comparison of robust DPO methods by the shift levels explicitly modeled in their objectives. \yes\ denotes a dedicated robustness component, while \no\ denotes no explicit level-specific component.}
\label{tab:method-comparison}
\setlength{\tabcolsep}{8pt}
\begin{tabular}{lcccc}
\toprule
Methods & Overall & Group-level & Prompt-level & Preference-level \\
\midrule
DPO~\citep{RafailovSMMEF23_DPO} & \no & \no & \no & \no \\
WDPO / KLDPO~\citep{xu2025neurips-robust_wdpo} & \yes & \no & \no & \no \\
GRPO~\citep{ramesh2024group_grpo} & \no & \yes & \no & \no \\
rDPO~\citep{chowdhury24a_rdpo} & \no & \no & \no & \yes \\
Dr.DPO~\citep{wu2025iclr-robust_drdpo} & \no & \no & \no & \yes \\
\midrule
\textbf{\method{} (Ours)} & \yes & \yes & \yes & \yes \\
\bottomrule
\end{tabular}
\vspace{-1.5em}
\end{table}

\section{Preliminaries}

\textbf{Notation.}
Unless otherwise specified, we use uppercase letters for random variables and lowercase letters for their realizations, with probability distributions denoted by $P$ and $Q$. We use $\boldsymbol{x}\in\mathcal{X}$ for a prompt, $\boldsymbol{y}\in\mathcal{Y}$ for a response, and $\boldsymbol{h}\in\mathcal{H}$ for the prompt embedding, with $\tilde{\boldsymbol{h}}=\boldsymbol{h}+\boldsymbol{\delta}$ denoting an adversarially perturbed embedding. A pairwise preference example is modeled as a random variable \(U=(X,V,S)\), where \(V=(Y_1,Y_2)\) is a pair of candidate responses and \(S\in\{+1,-1\}\) is the binary preference label indicating which response is preferred. In standard DPO training, each comparison \((\boldsymbol{x},\boldsymbol{y}^+,\boldsymbol{y}^-)\) is encoded as \(V=(\boldsymbol{y}^+,\boldsymbol{y}^-)\) and \(S=+1\). The group index is $g\in\mathcal{G}=[m]$, where $[m]=\{1,\dots,m\}$,  and $P_g$ denotes the group-specific distribution over $U$. 
Let $I_g$ be the index set of the $n_g=|I_g|$ training samples from group $g$.
We factorize $P_g(U)=P_g^X(X)\,P_g^V(V\mid X)\,P_g^S(S\mid V,X)$,  where $P_g^X$, $P_g^V$, and $P_g^S$ are the prompt marginal  distribution, conditional candidate response distribution, and  conditional preference distribution, respectively. We write $\Delta^{m-1}$ for the probability simplex over $m$ groups, $\|\boldsymbol{v}\|_2$ for the Euclidean norm, $\phi(\boldsymbol{h},\boldsymbol{y})\in\mathbb{R}^d$ for the log-linear feature map, and $\varphi(t)=\log(1+e^{-t})$ for the logistic loss.

\paragraph{Direct Preference Optimization.}
DPO learns a policy \(\pi_{\boldsymbol{\theta}}\) from pairwise preference data
\(\mathcal{D}=\{(\boldsymbol{x}_i,\boldsymbol{y}_i^+,\boldsymbol{y}_i^-)\}_{i=1}^n\), where \(\boldsymbol{y}_i^+\succ\boldsymbol{y}_i^-\) indicates that
\(\boldsymbol{y}_i^+\) is preferred over \(\boldsymbol{y}_i^-\). 
The DPO margin for a candidate pair \((\boldsymbol y_1,\boldsymbol y_2)\) is defined as
\begin{equation}
\label{eq:prelim-dpo-margin}
m(\boldsymbol{x},\boldsymbol{y}_1,\boldsymbol{y}_2;\boldsymbol{\theta})
=
\beta\!\left(
\log\frac{\pi_{\boldsymbol{\theta}}(\boldsymbol{y}_1\mid\boldsymbol{x})}
         {\pi_{\mathrm{ref}}(\boldsymbol{y}_1\mid\boldsymbol{x})}
-
\log\frac{\pi_{\boldsymbol{\theta}}(\boldsymbol{y}_2\mid\boldsymbol{x})}
         {\pi_{\mathrm{ref}}(\boldsymbol{y}_2\mid\boldsymbol{x})}
\right),
\end{equation}
where \(\beta>0\) is a scaling parameter and \(\pi_{\mathrm{ref}}\) is a reference policy.
Given a binary preference label \(s\in\{+1,-1\}\), the labeled DPO loss is $\mathcal{L}(\boldsymbol{x},\boldsymbol{y}_1,\boldsymbol{y}_2,s;\boldsymbol{\theta})=\varphi\!\left(s\cdot m(\boldsymbol{x},\boldsymbol{y}_1,\boldsymbol{y}_2;\boldsymbol{\theta})\right)$. For the standard DPO setting, we set \((\boldsymbol y_1,\boldsymbol y_2)=(\boldsymbol y^+,\boldsymbol y^-)\) and \(s=+1\), which gives the empirical DPO risk minimization objective: 
\begin{equation}
\label{eq:erm-objective}
\widehat{\mathcal R}_{\mathrm{DPO}}(\boldsymbol{\theta}):=\mathbb{E}_{\hat P}\bigl[\mathcal L(U;\boldsymbol{\theta})\bigr]=\frac{1}{n}\sum\nolimits_{i=1}^n
\varphi\bigl(
m(\boldsymbol{x}_i,\boldsymbol{y}_i^+,\boldsymbol{y}_i^-;\boldsymbol{\theta})
\bigr).
\end{equation}
Standard DPO minimizes this empirical risk under the assumption that training and deployment preferences follow the same distribution. When this distributional assumption fails, DPO can suffer from reduced preference alignment quality~\citep{wu2025iclr-robust_drdpo,xu2025neurips-robust_wdpo}.

\paragraph{Group DRO.}
Group DRO~\citep{SagawaKHL20_groupdro} introduces a group variable
$G\in\mathcal{G}=\{1,\dots,m\}$, so that each preference observation
$U=(X,V,S)$ is associated with a group index $g$.
Let $P_g$ denote the population distribution over $U$ for group $g$. The overall training
distribution is expressed as $P=\sum_{g=1}^m\alpha_g P_g$, where $\boldsymbol{\alpha}\in\Delta^{m-1}$
are the group proportions.
Through the lens of distributionally robust
optimization~\citep{Jiao2022Distributed,Esfahani2015DatadrivenDR}, Group DRO replaces the ERM objective in Eq.~\eqref{eq:erm-objective} with the following min-max optimization problem:
\begin{equation}
\label{eq:group-dro}
\inf_{\boldsymbol{\theta}\in\Theta}
\sup_{\boldsymbol{\omega}\in\mathcal{Q}}
\sum\nolimits_{g=1}^m\omega_g\,
\mathbb{E}_{P_g}\bigl[\mathcal{L}(U;\boldsymbol{\theta})\bigr],
\end{equation}
which minimizes the worst-case expected loss over an ambiguity set $\mathcal{Q}$ of group mixture weights, a paradigm that has also been adopted in robust RLHF~\citep{chakraborty2024maxminrlhf,ramesh2024group_grpo}. However, this formulation keeps the within-group distributions $P_g$ fixed and does not explicitly
model shifts in either the prompt marginal distribution
~\citep{li-etal-2023-robust,KohSMXZBHYPGLDS21}
or the conditional preference distribution
~\citep{SorensenMFGMRLJCD24_pluralistic,YaoCCYJYH25_gdpo}.
This motivates the hierarchical robust formulation
in Section~\ref{sec:hierarchical-shifts}.

\section{Method}

\subsection{Hierarchical Distribution Shifts and Uncertainty in Preference Alignment}
\label{sec:hierarchical-shifts}
In preference alignment, distribution shift exhibits a natural three-level hierarchical structure. At the \emph{global} level, the group proportions may shift between training and deployment~\citep{qiao2025gsr,ramesh2024group_grpo,jian2024trilevel}. Within each group, shift arises at the \emph{local} level in two forms: the prompt marginal distribution $P_g^X$ may shift across environments~\citep{li-etal-2023-robust,KohSMXZBHYPGLDS21}, and the conditional preference distribution $P_g^S(S\mid V,X)$ may vary due to noisy~\citep{chowdhury24a_rdpo} or pluralistic~\citep{PoddarWIGJ24,YaoCCYJYH25_gdpo} human feedback. We refer to these as \emph{group-level}, \emph{prompt-level}, and \emph{preference-level} distribution shifts, respectively. We next formalize these three levels of distributional uncertainty and construct a hierarchical ambiguity set for robust preference optimization.

\paragraph{Global group ambiguity set.}
At the global level, motivated by group DRO~\citep{SagawaKHL20_groupdro}, we model group-level shift by an ambiguity set over the group mixture weights, centered at the training group proportions $\boldsymbol{\alpha}$:
$\mathcal{U}^{\omega}= \bigl\{\boldsymbol{\omega}\in\Delta^{m-1}:\mathcal{D}_{\omega}(\boldsymbol{\omega},\boldsymbol{\alpha})\le\rho^{\omega}\bigr\}$,
where \(\mathcal{D}_{\omega}\) is a discrepancy measure on the group simplex and \(\rho^{\omega}\ge 0\) is the group-level radius.

\paragraph{Local rectangular ambiguity sets.}
At the local level, within-group shift arises from two sources: the prompt marginal distribution and the conditional preference distribution. For each group $g$, the shifted joint distribution factorizes as $Q_g(U)=Q_g^X(X)\,P_g^V(V\mid X)\,Q_g^S(S\mid X,V)$, where $Q_g^X$ and $Q_g^S$ are each constrained by their own ambiguity set. Since the candidate response pairs are fixed in the offline preference dataset~\citep{RafailovSMMEF23_DPO}, we do not introduce an ambiguity set for $P_g^V$ and keep $P_g^V(V\mid X)$ unchanged. We first constrain the prompt marginal distribution by $\mathcal{U}_g^X=\{Q_g^X:\mathcal{D}_X(Q_g^X,P_g^X)\le\rho_g^X\}$, where \(\mathcal{D}_X\) is a discrepancy measure on the prompt (embedding) space and \(\rho_g^X\ge 0\) is the prompt-level radius.

For each prompt and response pair \((\boldsymbol{x},\boldsymbol{v})\), we constrain the conditional preference distribution by
\begin{equation}
\label{eq:preference-ambiguity}
\mathcal{U}_{g,\boldsymbol{x},\boldsymbol{v}}^S
=
\Bigl\{
Q_g^S(\cdot\mid\boldsymbol{x},\boldsymbol{v}):
\mathcal{D}_S\bigl(
Q_g^S(\cdot\mid\boldsymbol{x},\boldsymbol{v}),
P_g^S(\cdot\mid\boldsymbol{x},\boldsymbol{v})
\bigr)
\le\rho_g^S
\Bigr\},
\end{equation}
where $\mathcal{D}_S$ measures discrepancies between distributions over binary preference labels, and $\rho_g^S\in[0,1]$ is the preference-level radius. Jointly, these constraints induce the local ambiguity set
\begin{equation}
\label{eq:local-ambiguity-set}
\mathcal{U}_g
=
\Bigl\{
Q_g:
Q_g^X\in\mathcal{U}_g^X,\;
Q_g^S(\cdot\mid\boldsymbol{x},\boldsymbol{v})
\in\mathcal{U}_{g,\boldsymbol{x},\boldsymbol{v}}^S
\ \text{for } Q_g^X P_g^V\text{-a.e. }(\boldsymbol{x},\boldsymbol{v})
\Bigr\}.
\end{equation}
Note that \(\mathcal{U}_g\) is \emph{rectangular} with respect to this factorization: the prompt marginal and the conditional preference distribution can be chosen separately under their respective constraints. Under standard measurability conditions, this rectangularity allows the adversary to first select a worst-case preference distribution for each fixed \((X,V)\), and then a worst-case prompt marginal distribution~\citep{Shapiro2016Rectangular,LiShapiro2025Rectangularity}, which yields the nested worst-case optimization in Section~\ref{sec:tractable-objective}.

\paragraph{Hierarchical ambiguity set.}
Combining the global and local constraints, the full ambiguity set is
\begin{equation}
\label{eq:structural-ambiguity}
\begin{array}{l}
\mathcal{U}
=
\Bigl\{
Q=\textstyle\sum_{g=1}^m \omega_g Q_g:
\boldsymbol{\omega}\in\mathcal{U}^{\omega},\;
Q_g\in\mathcal{U}_g\ \text{for all }g\in\mathcal{G}
\Bigr\}.
\end{array}
\end{equation}
Compared with group robustness methods that only consider group-level shift, our formulation more faithfully characterizes preference distribution shift by explicitly accounting for within-group variation in both the prompt marginal and the conditional preference distribution.

\subsection{Tri-level Distributionally Robust Preference Optimization}

Based on the hierarchical ambiguity set \(\mathcal{U}\) in Eq.~\eqref{eq:structural-ambiguity}, we now formulate the robust preference optimization problem. Since the true distribution is unknown but assumed to lie in \(\mathcal{U}\), we follow standard DRO principles~\citep{Wiesemann2014DistributionallyRC,Esfahani2015DatadrivenDR,Sinha2018Certifying} and minimize the worst-case risk over \(\mathcal{U}\):
\begin{equation}
\inf_{\boldsymbol{\theta}\in\Theta}
\sup_{\boldsymbol{\omega}\in\mathcal U^{\omega}}
\sum\nolimits_{g=1}^m
\omega_g
\sup_{Q_g\in\mathcal U_g}
\mathbb E_{U\sim Q_g}\big[\mathcal L(U;\boldsymbol{\theta})\big].
\label{eq:hpo-nested}
\end{equation}
The outer maximization over \(\boldsymbol{\omega}\) captures global group-level shift, and the inner maximization over \(Q_g\) captures local prompt- and preference-level shifts. To solve the inner maximization, we first define the worst-case preference loss of group \(g\) at a fixed prompt and response pair \((\boldsymbol{x},\boldsymbol{v})\) as
\begin{equation}
\psi_g(\boldsymbol{x},\boldsymbol{v};\boldsymbol{\theta})
:=
\sup_{Q_g^S\in\mathcal U_{g,\boldsymbol{x},\boldsymbol{v}}^S}
\mathbb E_{S\sim Q_g^S}
\big[
\mathcal L(\boldsymbol{x},\boldsymbol{v},S;\boldsymbol{\theta})
\big].
\label{eq:conditional-robust-loss}
\end{equation}
Since \(\mathcal U_g\) is rectangular, the adversarial choices of the prompt marginal \(Q_g^X\) and the conditional preference kernel \(Q_g^S(\cdot\mid X,V)\) are decoupled, which enables a nested decomposition of the worst-case expectation over \(Q_g\)~\citep{Shapiro2016Rectangular,LiShapiro2025Rectangularity}. Consequently, Eq.~\eqref{eq:hpo-nested} can be rewritten as
\begin{equation}
\label{eq:hpo-final}
\inf_{\boldsymbol{\theta}\in\Theta}
\sup_{\boldsymbol{\omega}\in\mathcal{U}^{\omega}}
\sum\nolimits_{g=1}^{m}
\omega_g
\sup_{Q_g^X\in\mathcal{U}_g^X}
\mathbb{E}_{X\sim Q_g^X,\; V\sim P_g^V(\cdot\mid X)}
\bigl[
\psi_g(X,V;\boldsymbol{\theta})
\bigr].
\end{equation}
The above nested optimization captures the three
sources of distribution shift described in
Section~\ref{sec:hierarchical-shifts}. It generalizes Group DRO by accounting for within-group
distributional uncertainty: when $\rho_g^X=\rho_g^S=0$
for all groups, the formulation reduces to a Group DRO
objective over $\mathcal{U}^{\omega}$.
While nested optimization structures have been studied
in bilevel learning~\citep{franceschi2018bilevel,liu2025nested,jiao2022asynchronous,ji2021bilevel} and trilevel learning~\citep{jiaodtzo,jian2024trilevel,yang2026stcor},
our formulation arises from hierarchical distributional uncertainty in preference alignment.

\subsection{Finite-Sample Guarantees}
We establish a finite-sample convergence guarantee for the fixed-penalty Wasserstein dual objective of \method{}, whose derivation is presented in Section~\ref{sec:tractable-objective}. The analysis relies on the following two assumptions. 

\begin{assumption}[Log-linear policy class~\citep{xu2025neurips-robust_wdpo} and boundedness~\citep{jiao2024provably}]
\label{ass:loglinear-bounded}
Let \(\phi:\mathcal H\times\mathcal Y\to\mathbb R^d\) be a feature map satisfying \(\sup_{\boldsymbol h\in\mathcal H,\,\boldsymbol y\in\mathcal Y}\|\phi(\boldsymbol h,\boldsymbol y)\|_2\le C_\phi\), where \(\mathcal H\) includes both clean and perturbed prompt embeddings. Let \(\Theta=\{\boldsymbol \theta\in\mathbb R^d:\|\boldsymbol \theta\|_2\le B\}\), and \(\pi_{\mathrm{ref}}=\pi_{\boldsymbol\theta_{\rm ref}}\) satisfy \(\|\boldsymbol \theta_{\rm ref}\|_2\le B\). We consider the log-linear policy class
\[
\Pi
=
\left\{
\pi_{\boldsymbol \theta}:
\pi_{\boldsymbol \theta}(y\mid \boldsymbol h)
=
\frac{\exp(\boldsymbol \theta^\top\phi(\boldsymbol h,\boldsymbol y))}
{\sum_{\boldsymbol{y'}\in\mathcal Y}\exp(\boldsymbol \theta^\top\phi(\boldsymbol h,\boldsymbol y'))},
\ \boldsymbol \theta\in\Theta
\right\}.
\]
\end{assumption}

\begin{assumption}[Uniform covariance coverage~\citep{zhu2023principled}]
\label{ass:uniform-coverage}
Let $z(\boldsymbol h,\boldsymbol y^+\!,\boldsymbol y^-)\!=\!\phi(\boldsymbol h,\boldsymbol y^+)-\phi(\boldsymbol h,\boldsymbol y^-)$.
For each group $g$ and $\boldsymbol\theta\in\Theta$, assume the inner maximization in Eq.~\eqref{eq:dual-prompt-perturb} admits a measurable maximizer $\boldsymbol\delta_{g,\boldsymbol\theta}^\star(\boldsymbol h,\boldsymbol y^+,\boldsymbol y^-)$, and let $\tilde{\boldsymbol h}=\boldsymbol h+\boldsymbol\delta_{g,\boldsymbol\theta}^\star(\boldsymbol h,\boldsymbol y^+,\boldsymbol y^-)$.
There exists $\lambda>0$ such that, for all $\boldsymbol\theta\in\Theta$, $\boldsymbol\omega\in\mathcal U^\omega$, and all such maximizers,
$
\sum_{g=1}^m \omega_g\,
\mathbb E_{P_g}\!\left[
z(\tilde{\boldsymbol h},\boldsymbol y^+,\boldsymbol y^-)\,
z(\tilde{\boldsymbol h},\boldsymbol y^+,\boldsymbol y^-)^\top
\right]
\succeq \lambda I_d.
$
\end{assumption}
For each group $g$, we fix $\gamma_g>0$ in Eq.~\eqref{eq:dual-local-risk} and define the empirical risk as $\widehat{\mathcal R}_g(\boldsymbol\theta)=\gamma_g\rho_g^X+\frac{1}{n_g}\sum_{i\in I_g}\ell_{\boldsymbol\theta}^g(\boldsymbol z_i;\gamma_g)$, where $\ell_{\boldsymbol\theta}^g$ is given by Eq.~\eqref{eq:dual-prompt-perturb}. The corresponding population risk $\mathcal R_g(\boldsymbol\theta)$ replaces the empirical average with $\mathbb E_{P_g}$. The hierarchical population and empirical objectives are $\mathcal R_{\mathrm{\method{}}}(\boldsymbol\theta)=\sup_{\boldsymbol\omega\in\mathcal U^\omega}\sum_{g=1}^m\omega_g\mathcal R_g(\boldsymbol\theta)$ and $\widehat{\mathcal R}_{\mathrm{\method{}}}(\boldsymbol\theta)=\sup_{\boldsymbol\omega\in\mathcal U^\omega}\sum_{g=1}^m\omega_g\widehat{\mathcal R}_g(\boldsymbol\theta)$, respectively. Let $\boldsymbol\theta^\star\in\arg\min_{\boldsymbol\theta\in\Theta}\mathcal R_{\mathrm{\method{}}}(\boldsymbol\theta)$ and $\widehat{\boldsymbol\theta}\in\arg\min_{\boldsymbol\theta\in\Theta}\widehat{\mathcal R}_{\mathrm{\method{}}}(\boldsymbol\theta)$ denote the population and empirical minimizers, respectively. The following theorem establishes a convergence rate governed by the smallest group sample size $n_{\min}=\min_{g\in[m]}n_g$.
\begin{restatable}[Finite-sample convergence of \method{}]{theorem}{HPOconvergence}
\label{thm:hpo-main}
Let \(M_\beta = 4\beta B C_\phi\) and \(\kappa_\beta = \frac{\beta^2 e^{M_\beta}}{(1+e^{M_\beta})^2}\). Suppose Assumptions~\ref{ass:loglinear-bounded} and~\ref{ass:uniform-coverage} hold. With probability at least \(1-\delta\), we have
\[
\|\widehat{\boldsymbol \theta}-\boldsymbol \theta^\star\|_2
\le
2
\sqrt{
\frac{
\max_{g\in[m]}
\left[
2A_g\sqrt{\frac{d}{n_g}}
+
K_g
\sqrt{\frac{\log(2m/\delta)}{2n_g}}
\right]
}{
\kappa_\beta\lambda
}
},
\]
where $K_g=\log(1+e^{M_\beta})+\rho_g^S M_\beta$,
$L_g=2\beta C_\phi(1+\rho_g^S)$,
$A_g=C K_g\sqrt{\log\!\big(e+3BL_g/K_g\big)}$,
and $C>0$ is a universal constant. The proof is in Appendix~\ref{app:proof-theorem1}.
\end{restatable}

\section{Implementation Algorithm}
\subsection{Wasserstein Instantiation}

To obtain a tractable learning objective, we instantiate the discrepancy measures $\mathcal{D}_{\omega}$, $\mathcal{D}_X$, and $\mathcal{D}_S$ with the Wasserstein distance~\citep{Wiesemann2014DistributionallyRC,Esfahani2015DatadrivenDR,Sinha2018Certifying}.

\begin{definition}[Wasserstein Distance]
Let $P$ and $Q$ be probability measures on a measurable space $\mathcal{Z}$ equipped with a cost function $c:\mathcal{Z}\times\mathcal{Z}\to[0,\infty)$.
The Wasserstein distance induced by $c$ is
\begin{equation}
\label{eq:wasserstein-distance}
\mathcal{W}_c(P,Q)
:=
\inf_{\xi\in\Gamma(P,Q)}
\int_{\mathcal{Z}\times\mathcal{Z}} c(u,u')\,\mathrm{d}\xi(u,u'),
\end{equation}
where $\Gamma(P,Q)$ denotes the set of all couplings of $P$ and $Q$.
\end{definition}

At the global level, we model uncertainty in the group mixture weights over the finite set $\mathcal{G}=\{1,\dots,m\}$. Since the groups have no natural geometry, we equip $\mathcal{G}$ with the discrete cost $d_G(g,g')=\mathbf{1}[g\neq g']$. Under this cost, the Wasserstein distance reduces to the total variation distance, i.e., $\mathcal{W}_{d_G}(\boldsymbol{\omega},\boldsymbol{\alpha})=\frac{1}{2}\|\boldsymbol{\omega}-\boldsymbol{\alpha}\|_1$.
At the local level, we model within-group uncertainty through Wasserstein balls:
\[
\mathcal W_{c_X}(Q_g^X,P_g^X)\le \rho_g^X,
\qquad
\mathcal W_{c_S}\!\left(
Q_g^S(\cdot\mid\boldsymbol{x},\boldsymbol{v}),
P_g^S(\cdot\mid\boldsymbol{x},\boldsymbol{v})
\right)\le \rho_g^S.
\]
Here, \(c_X\) is the transport cost on the prompt space and \(c_S\) is the label-flip cost on \(\{+1,-1\}\). The radii \(\rho_g^X\) and \(\rho_g^S\) independently control prompt-level and preference-level robustness within group \(g\).

\subsection{Tractable Objective}
\label{sec:tractable-objective}

In this section, we derive a tractable reformulation by successively robustifying the conditional preference distribution, the prompt marginal distribution, and the distribution over group weights.
\paragraph{Preference-level robustification.}
Recall that, for group \(g\), the conditional distribution factorizes as $P_g(V,S\mid X)=P_g^V(V\mid X)\,P_g^S(S\mid X,V)$. Since the candidate response pair \(V\) is observed and fixed in the offline dataset, the remaining preference-level uncertainty lies in the conditional distribution over binary labels \(P_g^S(S\mid X,V)\). Thus, we model preference-level shift as a perturbation of this binary label distribution, holding the response pair fixed. In standard DPO data, each comparison is encoded with $S=+1$, so $P_g^S(\cdot\mid\boldsymbol{x},\boldsymbol{v})=\delta_{+1}$. For each group \(g\) and each prompt and response pair \((\boldsymbol{x},\boldsymbol{v})=(\boldsymbol{x},\boldsymbol{y}_1,\boldsymbol{y}_2)\), we define
\begin{equation}
\label{eq:pref-label-ball}
\mathcal{U}_{g,\boldsymbol{x},\boldsymbol{v}}^S
=
\Bigl\{q\in\Delta(\{+1,-1\}):
\mathcal{W}_{c_S}(q,\delta_{+1})\le\rho_g^S\Bigr\},
\end{equation}
with the label-flip cost $c_S(s,s')=\mathbf{1}[s\neq s']$. Under this cost, $\mathcal{W}_{c_S}(q,\delta_{+1})$ equals the probability of $s=-1$ under $q$, so the ball in Eq.~\eqref{eq:pref-label-ball} allows at most a fraction $\rho_g^S$ of the label mass to be flipped from $+1$ to $-1$. The worst-case preference-level robust loss is then
\begin{equation}
\label{eq:pref-robust-sample-loss}
\bar{\mathcal{L}}_g(\boldsymbol{x},\boldsymbol{y}_1,\boldsymbol{y}_2;\boldsymbol{\theta})
:=
\sup_{q\in\mathcal{U}_{g,\boldsymbol{x},\boldsymbol{v}}^S}
\mathbb{E}_{S\sim q}\bigl[\mathcal{L}(\boldsymbol{x},\boldsymbol{y}_1,\boldsymbol{y}_2,S;\boldsymbol{\theta})\bigr].
\end{equation}

\begin{proposition}[Closed-form preference-level robust loss]
\label{prop:closed-form-robust-loss}
For any group $g$, prompt and candidate response pair $(\boldsymbol{x},\boldsymbol{y}_1,\boldsymbol{y}_2)$, we have 
\begin{equation}
\label{eq:closed-form-psi}
\bar{\mathcal{L}}_g(\boldsymbol{x},\boldsymbol{y}_1,\boldsymbol{y}_2;\boldsymbol{\theta})
=
\varphi\bigl(m(\boldsymbol{x},\boldsymbol{y}_1,\boldsymbol{y}_2;\boldsymbol{\theta})\bigr)
+
\rho_g^S\bigl[m(\boldsymbol{x},\boldsymbol{y}_1,\boldsymbol{y}_2;\boldsymbol{\theta})\bigr]_+,
\end{equation}
where $[t]_+=\max\{t,0\}$. The proof is provided in Appendix~\ref{app:proof-prop1}.
\end{proposition}
\begin{proposition}[Optimal robust preference margin]
\label{prop:optimal-margin}
For any group \(g\), let \(M=m(\boldsymbol{x},\boldsymbol{y}_1,\boldsymbol{y}_2;\boldsymbol{\theta})\) denote the DPO margin defined in Eq.~\eqref{eq:prelim-dpo-margin}. Then $\mathcal{L}_g(M)=\varphi(M)+\rho_g^S[M]_+$ is minimized at
\begin{equation}
\label{eq:optimal-margin}
M_g^\star
=
\begin{cases}
+\infty, & \rho_g^S=0,\\[4pt]
\log\dfrac{1-\rho_g^S}{\rho_g^S}, & 0<\rho_g^S<\tfrac{1}{2},\\[8pt]
0, & \rho_g^S\ge\tfrac{1}{2}.
\end{cases}
\end{equation}
The proof is provided in Appendix~\ref{app:proof-prop2}.
\end{proposition}

\begin{remark}
Proposition~\ref{prop:closed-form-robust-loss} shows that preference-level Wasserstein robustness yields a robustified loss that appends a one-sided margin penalty scaled by $\rho_g^S$.
Proposition~\ref{prop:optimal-margin} indicates that, unlike standard DPO, preference-level Wasserstein robustness does not drive the margin to grow unboundedly. Instead, it yields a finite optimal target margin. Moreover, this target margin decreases as $\rho_g^S$ increases, implying that stronger uncertainty in the conditional preference distribution yields a more conservative optimum and reduces overconfidence in weak or potentially misaligned preferences.
\end{remark}

\paragraph{Prompt-level robustification.}
We next robustify the prompt marginal distribution. For each group \(g\), the local robust risk is
\begin{equation}
\label{eq:group-local-risk}
\mathcal{R}_g(\boldsymbol{\theta})
:=
\sup_{Q_g^X\in\mathcal{U}_g^X}
\mathbb{E}_{X\sim Q_g^X,\; V\sim P_g^V(\cdot\mid X)}
\bigl[\psi_g(X,V;\boldsymbol{\theta})\bigr],
\end{equation}
where \(\psi_g(\boldsymbol{x},\boldsymbol{v};\boldsymbol{\theta})\) is the worst-case preference loss in Eq.~\eqref{eq:conditional-robust-loss}. Under the label-flip Wasserstein ball in Eq.~\eqref{eq:pref-label-ball}, \(\psi_g\) admits the closed form \(\bar{\mathcal{L}}_g(\boldsymbol{x},\boldsymbol{y}_1,\boldsymbol{y}_2;\boldsymbol{\theta})\) given in Proposition~\ref{prop:closed-form-robust-loss}.

To obtain a tractable empirical formulation, we follow the robust NLP fine-tuning literature and perturb prompts in the embedding space rather than the discrete token space~\citep{miyato2017adversarial,zhu2020freelb,jiang-etal-2020-smart}. Let $\boldsymbol{h}=\mathcal{E}(\boldsymbol{x})\in\mathbb{R}^d$ denote the prompt embedding and $\boldsymbol z_i=(\boldsymbol{h}_i,\boldsymbol{y}_i^+,\boldsymbol{y}_i^-)$ the $i$-th embedded observation. For group $g$, the empirical distribution over embedded preference observations is $\widehat P_g=\frac{1}{n_g}\sum_{i\in I_g}\delta_{\boldsymbol z_i}$.

Following Wasserstein DRO with local input perturbations~\citep{kwon20a} and embedding-space adversarial training in NLP~\citep{zhu2020freelb,miyato2017adversarial,jiang-etal-2020-smart}, we perturb the prompt embedding of each preference observation while keeping the associated response pair fixed. 
Accordingly, we define the transport cost
\begin{equation}
\label{eq:prompt-transport-cost}
c_X\bigl((\boldsymbol{h},\boldsymbol{y}^+,\boldsymbol{y}^-),
          (\boldsymbol{h}',\tilde{\boldsymbol{y}}^+,\tilde{\boldsymbol{y}}^-)\bigr)
=
\begin{cases}
\frac{1}{2}\|\boldsymbol{h}-\boldsymbol{h}'\|_2^2,
  & (\boldsymbol{y}^+,\boldsymbol{y}^-)=(\tilde{\boldsymbol{y}}^+,\tilde{\boldsymbol{y}}^-),\\
+\infty, & \text{otherwise.}
\end{cases}
\end{equation}
For this empirical approximation, we define the local robust risk as
\begin{equation}
\label{eq:empirical-local-risk}
\widehat{\mathcal{R}}_g(\boldsymbol{\theta})
=
\sup_{\widetilde{Q}_g^X:\,\mathcal{W}_{c_X}(\widetilde{Q}_g^X,\widehat{P}_g)\le\rho_g^X}
\mathbb{E}_{(\boldsymbol{h},\boldsymbol{y}^+,\boldsymbol{y}^-)\sim\widetilde{Q}_g^X}
\bigl[\bar{\mathcal{L}}_g(\boldsymbol{h},\boldsymbol{y}^+,\boldsymbol{y}^-;\boldsymbol{\theta})\bigr].
\end{equation}
\vspace{-0.5em}
\begin{lemma}[Wasserstein DRO duality~\citep{Sinha2018Certifying,ZhangYangGao2025}]
\label{lem:wasserstein-duality}
Under standard regularity conditions for Wasserstein DRO duality, for any reference distribution $Q$, radius $\rho\ge0$, loss $\ell$, and transport cost $c$, we have
\begin{equation}
\label{eq:wasserstein-duality}
\sup_{P:\,\mathcal{W}_c(P,Q)\le\rho}
\mathbb{E}_P[\ell(Z)]
=
\inf_{\gamma\ge 0}
\Bigl\{
\gamma\rho
+\mathbb{E}_Q\bigl[\ell_p(Z;\gamma)\bigr]
\Bigr\},
\end{equation}
where $\ell_p(u_0;\gamma)=\sup_{u}\{\ell(u)-\gamma\,c(u,u_0)\}$ is the Wasserstein dual surrogate loss.
\end{lemma}
Applying Lemma~\ref{lem:wasserstein-duality} to Eq.~\eqref{eq:empirical-local-risk} with $\ell=\bar{\mathcal{L}}_g$ and $c=c_X$ yields the dual form
\begin{equation}
\label{eq:dual-local-risk}
\widehat{\mathcal{R}}_g(\boldsymbol{\theta})
=
\inf_{\gamma_g\ge 0}
\left\{
\gamma_g\rho_g^X
+\frac{1}{n_g}\sum_{i\in I_g}
\ell_{\boldsymbol{\theta}}^{g}(\boldsymbol z_i;\gamma_g)
\right\},
\end{equation}
where \(\boldsymbol z_i=(\boldsymbol{h}_i,\boldsymbol{y}_i^+,\boldsymbol{y}_i^-)\). Under the transport cost in Eq.~\eqref{eq:prompt-transport-cost}, changing the response pair incurs infinite cost, so the inner maximization only searches over perturbations of the prompt embedding:
\begin{equation}
\label{eq:dual-prompt-perturb}
\ell_{\boldsymbol{\theta}}^{g}(\boldsymbol z_i;\gamma_g)
=
\sup_{\boldsymbol{\delta}\in\mathbb{R}^d}
\left\{
\bar{\mathcal{L}}_g(\boldsymbol{h}_i+\boldsymbol{\delta},
\boldsymbol{y}_i^+,\boldsymbol{y}_i^-;\boldsymbol{\theta})
-\frac{\gamma_g}{2}\|\boldsymbol{\delta}\|_2^2
\right\}.
\end{equation}

\paragraph{Final empirical objective.}
Combining the global Wasserstein reweighting with the dual form of the local risks, the final empirical \method{} objective is
\begin{equation}
\label{eq:empirical-hpo}
\min_{\boldsymbol{\theta}\in\Theta}
\sup_{\boldsymbol{\omega}\in\mathcal{U}^{\omega}}
\sum\nolimits_{g=1}^{m}
\omega_g\,\widehat{\mathcal{R}}_g(\boldsymbol{\theta}),
\end{equation}
where $\widehat{\mathcal{R}}_g(\boldsymbol{\theta})$ is given in Eq.~\eqref{eq:dual-local-risk} and $\bar{\mathcal{L}}_g$ in Eq.~\eqref{eq:closed-form-psi}. Eq.~\eqref{eq:empirical-hpo} reduces the abstract hierarchical Wasserstein robust objective to a concrete min-max training problem.

\subsection{Optimization Algorithm}

While the dual reformulation in Eq.~\eqref{eq:dual-prompt-perturb} is exact,
explicitly optimizing over $\gamma_g$ at each training step is computationally
prohibitive. Following standard practice in Wasserstein robust
optimization~\citep{Sinha2018Certifying,ZhangYYXGJ19,kwon20a,BaiHJO25}, we treat
$\gamma_g$ as a fixed hyperparameter and additionally constrain each perturbation to
$\mathcal{B}_g^X=\{\boldsymbol{\delta}:
\|\boldsymbol{\delta}\|_2^2\le 2\rho_g^X\}$. For each group $g$, the resulting
penalty-based surrogate objective is
\begin{equation}
\label{eq:practical-group-risk}
\widetilde{\mathcal{R}}_g(\boldsymbol{\theta};\gamma_g)
:=
\gamma_g\rho_g^X +
\frac{1}{n_g}\sum\nolimits_{i\in I_g}
\sup_{\boldsymbol{\delta}\in\mathcal{B}_g^X}
\left[
\bar{\mathcal{L}}_g(
\boldsymbol{h}_i+\boldsymbol{\delta},
\boldsymbol{y}_i^+,\boldsymbol{y}_i^-;
\boldsymbol{\theta})
-\frac{\gamma_g}{2}\|\boldsymbol{\delta}\|_2^2
\right].
\end{equation}
The full training procedure is summarized in Algorithm~\ref{alg:hdrpo}.
Following alternating optimization strategies in adversarial training~\citep{MadryMSTV18,Sinha2018Certifying,jiao2025prattack}
and robust group reweighting~\citep{SagawaKHL20_groupdro}, each iteration proceeds in three stages.
\textbf{(i) Adversarial Perturbation.}
For each embedded preference observation
$(\boldsymbol{h}_i,\boldsymbol{y}_i^+,\boldsymbol{y}_i^-)$ in the mini-batch,
we approximate the prompt-level supremum in
Eq.~\eqref{eq:practical-group-risk} by running $K$ steps of projected gradient ascent on
$\boldsymbol{\delta}$ with the fixed quadratic penalty.
\textbf{(ii) Group-weight update.}
Using the perturbed embeddings from stage (i), we compute group-wise stochastic
estimates $\widetilde{\mathcal{R}}_g^{(t)}$ and perform one step of exponentiated
gradient ascent projected onto $\mathcal{U}^{\omega}$, which implements the
worst-case group reweighting over the $L_1$-constrained simplex.
\textbf{(iii) Policy update.} We update the policy parameters $\boldsymbol{\theta}$
by taking a gradient descent step on the weighted
surrogate objective
$\sum_{g=1}^m \omega_g^{(t)}
\widetilde{\mathcal{R}}_g^{(t)}$,
with the adversarial perturbations and group weights fixed. 

\begin{algorithm}[t]
\caption{ Alternating Optimization for \method{}}
\label{alg:hdrpo}

\begin{algorithmic}[1]
\State \textbf{Input:} Step sizes $\eta_{\theta},\eta_{\omega},\eta_{\delta}$;
         penalty parameters $\gamma_g$, 
         inner steps $K$, and total iterations $T$.
\For{$t=1,\dots,T$}
    \State Sample group mini-batches
           $\mathcal{B}^{(t)}=\{\mathcal{B}_g^{(t)}\}_{g=1}^m$,
           $\mathcal{B}_g^{(t)}\subseteq\mathcal{D}_g$
    \For{each $g\in\mathcal{G}$ and
         $(\boldsymbol{h}_i,\boldsymbol{y}_i^+,\boldsymbol{y}_i^-)\in\mathcal{B}_g^{(t)}$}
        \State Initialize $\boldsymbol{\delta}_{g,i}^{(0)}\gets\boldsymbol{0}$
        \For{$k=0,\dots,K-1$}
            \State $\boldsymbol{\delta}_{g,i}^{(k+1)}
            \gets \Pi_{\mathcal{B}_g^X}\!\left(
            \boldsymbol{\delta}_{g,i}^{(k)}
            +\eta_{\delta}
            \nabla_{\boldsymbol{\delta}}\!\left[
            \bar{\mathcal{L}}_g\!\left(
            \boldsymbol{h}_i+\boldsymbol{\delta},
            \boldsymbol{y}_i^+,\boldsymbol{y}_i^-;
            \boldsymbol{\theta}^{(t-1)}
            \right)
            -\tfrac{\gamma_g}{2}\|\boldsymbol{\delta}\|_2^2
            \right]_{\boldsymbol{\delta}=\boldsymbol{\delta}_{g,i}^{(k)}}
            \right)$
        \EndFor
    \EndFor
    \For{each $g\in\mathcal{G}$}
        \State $\widetilde{\mathcal{R}}_g^{(t)}(\boldsymbol{\theta})\gets
               \gamma_g\rho_g^X
               +\dfrac{1}{|\mathcal{B}_g^{(t)}|}
               \displaystyle\sum\nolimits_{i\in\mathcal{B}_g^{(t)}}
               \!\left[
               \bar{\mathcal{L}}_g\!\left(\boldsymbol{h}_i+\boldsymbol{\delta}_{g,i}^{(K)},
               \boldsymbol{y}_i^+,\boldsymbol{y}_i^-;\boldsymbol{\theta}\right)
               -\tfrac{\gamma_g}{2}\|\boldsymbol{\delta}_{g,i}^{(K)}\|_2^2
               \right]$
    \EndFor
    \State $\boldsymbol{\omega}^{(t)}\gets
           \Pi_{\mathcal{U}^{\omega}}\!\left(
           \boldsymbol{\omega}^{(t-1)}\odot
           \exp\bigl(\eta_{\omega}\widetilde{\mathcal{R}}^{(t)}(\boldsymbol{\theta}^{(t-1)}\bigr)
           \right)$
    \State $\boldsymbol{\theta}^{(t)}\gets\boldsymbol{\theta}^{(t-1)}
           -\eta_{\theta}\,
           \left.
           \nabla_{\boldsymbol{\theta}}\!
           \left(\sum_{g=1}^m\omega_g^{(t)}\widetilde{\mathcal{R}}_g^{(t)}(\boldsymbol{\theta})\right)
           \right|_{\boldsymbol{\theta}=\boldsymbol{\theta}^{(t-1)}}$
\EndFor
\State \textbf{Return:} $\boldsymbol{\theta}^{(T)}$
\end{algorithmic}
\end{algorithm}

\begin{table*}[t]
\centering
\small
\caption{Average and worst-group preference accuracy (\%) under different OOD shifts. The best result is in \textbf{bold} and the second-best is \underline{underlined}.}
\label{tab:emotion_avg_worst_std_maintext}

\centering

\resizebox{\linewidth}{!}{
\begin{tabular}{lccccccccccc}
\toprule
\multirow{2}{*}{Qwen-2.5-0.5B}
& \multicolumn{2}{c}{ID}
& \multicolumn{2}{c}{Group}
& \multicolumn{2}{c}{Prompt}
& \multicolumn{2}{c}{Preference}
& \multicolumn{2}{c}{Combined}
& \multirow{2}{*}{Avg. Rank} \\
\cmidrule(lr){2-3}
\cmidrule(lr){4-5}
\cmidrule(lr){6-7}
\cmidrule(lr){8-9}
\cmidrule(lr){10-11}
& Avg Acc & Worst Acc
& Avg Acc & Worst Acc
& Avg Acc & Worst Acc
& Avg Acc & Worst Acc
& Avg Acc & Worst Acc & \\
\midrule
DPO
& \sE{\underline{86.87}} & \sA{77.78}
& \sA{81.80} & \sA{76.15}
& \sA{83.17} & \sA{76.98}
& \sB{82.93} & \sA{78.16}
& \sB{82.60} & \sC{76.40} & \sB{5.30} \\

Dr.DPO
& \sA{84.81} & \sC{81.36} 
& \sA{81.65} & \sB{77.06}
& \sA{82.75} & \sA{76.84}
& \sA{82.24} & \sA{75.86}
& \sA{81.27} & \sA{75.22} & \sA{6.50} \\

WDPO
& \sA{84.87} & \sE{\underline{84.13}}
& \sB{82.72} & \sD{78.90}
& \sE{\underline{85.08}} & \sE{\underline{82.20}}
& \sD{84.01} & \sD{81.46}
& \sD{84.21} & \sD{78.76} & \sD{3.20} \\

\method{} (Group)
& \sB{84.90} & \sA{77.27}
& \sE{\underline{82.96}} & \sA{76.76}
& \sB{83.46} & \sB{78.28}
& \sA{82.31} & \sB{79.24}
& \sA{82.06} & \sB{76.11} & \sB{5.25} \\

\method{} (Prompt)
& \sD{85.93} & \sD{83.77}
& \sD{82.84} & \sE{\underline{80.43}}
& \sD{84.69} & \sD{81.05}
& \sE{\underline{85.02}} & \sF{\textbf{83.03}}
& \sE{\underline{84.37}} & \sE{\underline{79.65}} & \sE{\underline{2.40}} \\

\method{} (Preference)
& \sC{85.64} & \sB{78.28}
& \sC{82.77} & \sC{77.98}
& \sC{84.30} & \sC{79.75}
& \sC{83.28} & \sC{79.99}
& \sC{83.24} & \sB{76.11} & \sC{4.25} \\

\method{}
& \sF{\textbf{86.95}} & \sF{\textbf{85.12}}
& \sF{\textbf{84.61}} & \sF{\textbf{81.65}}
& \sF{\textbf{85.47}} & \sF{\textbf{82.56}}
& \sF{\textbf{85.44}} & \sE{\underline{82.34}}
& \sF{\textbf{85.25}} & \sF{\textbf{79.94}} & \sF{\textbf{1.10}} \\
\bottomrule
\end{tabular}
}

\begin{minipage}{\linewidth}
\centering

\resizebox{\linewidth}{!}{
\begin{tabular}{lccccccccccc}
\toprule
\multirow{2}{*}{LLaMA-3.2-1B}
& \multicolumn{2}{c}{ID}
& \multicolumn{2}{c}{Group}
& \multicolumn{2}{c}{Prompt}
& \multicolumn{2}{c}{Preference}
& \multicolumn{2}{c}{Combined} 
& \multirow{2}{*}{Avg. Rank} \\
\cmidrule(lr){2-3}
\cmidrule(lr){4-5}
\cmidrule(lr){6-7}
\cmidrule(lr){8-9}
\cmidrule(lr){10-11}
& Avg Acc & Worst Acc
& Avg Acc & Worst Acc
& Avg Acc & Worst Acc
& Avg Acc & Worst Acc
& Avg Acc & Worst Acc & \\
\midrule
DPO
& \sB{88.75} & \sC{86.19}
& \sB{88.23} & \sE\underline{83.44}
& \sB{87.81} & \sB{85.43}
& \sB{86.91} & \sF{\textbf{81.42}}
& \sB{87.44} & \sD{82.61} & \sB{4.15}\\

Dr.DPO
& \sA{84.16} & \sA{81.00}
& \sA{84.65} & \sB{79.87}
& \sA{83.73} & \sA{80.98}
& \sA{82.87} & \sA{78.87}
& \sA{85.51} & \sA{79.13} & \sA{7.00}\\

WDPO
& \sA{87.53} & \sB{84.17}
& \sA{87.24} & \sC{80.86}
& \sA{87.78} & \sC{85.82}
& \sA{85.28} & \sB{79.73}
& \sA{86.71} & \sB{81.16} & \sA{5.80}\\

\method{} (Group)
& \sC{89.03} & \sE{\underline{87.02}}
& \sD{88.61} & \sE{83.39}
& \sC{88.32} & \sD{85.86}
& \sC{86.94} & \sD{81.22}
& \sC{87.54} & \sE{\underline{82.90}} & \sD{3.50}\\

\method{} (Prompt)
& \sE{\underline{89.78}} & \sD{86.37}
& \sC{88.50} & \sD{81.19}
& \sD{88.71} & \sC{85.41}
& \sD{87.30} & \sE{\underline{81.31}}
& \sD{87.63} & \sC{82.03} & \sC{3.75}\\

\method{} (Preference)
& \sD{89.75} & \sF{\textbf{87.55}}
& \sE{\underline{89.00}} & \sE{83.17}
& \sE{\underline{89.00}} & \sE{\underline{86.23}}
& \sE{\underline{87.40}} & \sE{81.24}
& \sE{\underline{87.87}} & \sD{82.61} & \sE{\underline{2.60}}\\

\method{}
& \sF{\textbf{90.57}} & \sF{\textbf{87.55}}
& \sF{\textbf{89.11}} & \sF{\textbf{83.50}}
& \sF{\textbf{89.14}} & \sF{\textbf{86.70}}
& \sF{\textbf{88.25}} & \sE{\underline{81.31}}
& \sF{\textbf{88.31}} & \sF{\textbf{83.19}} & \sF{\textbf{1.20}}\\
\bottomrule
\end{tabular}

}
\end{minipage}
\vspace{-1em}
\end{table*}

\section{Experiments}
\label{sec:experiment}

\subsection{Experimental Setup}
We evaluate \method{} on three tasks: OOD preference alignment, preference poisoning, and downstream generalization. For OOD preference alignment and downstream generalization, we compare against standard DPO and several recent robust baselines: GRPO~\citep{ramesh2024group_grpo}, Dr.DPO~\citep{wu2025iclr-robust_drdpo}, and WDPO~\citep{xu2025neurips-robust_wdpo}. We also include three variants of \method{}: \method{} (Group), \method{} (Prompt), and \method{} (Preference), which activate only the group-level, prompt-level, and preference-level robustness components, respectively. Note that \method{} (Group) reduces to GRPO~\citep{ramesh2024group_grpo} as a special case. For preference poisoning, we compare against rDPO~\citep{chowdhury24a_rdpo} and Dr.DPO~\citep{wu2025iclr-robust_drdpo}, two robust DPO baselines designed to mitigate noisy preferences.

\paragraph{OOD Preference Alignment.}
Following prior emotion-based evaluations of preference shift~\citep{xu2025neurips-robust_wdpo,sahu2025online}, we construct a benchmark from the Emotion dataset~\citep{saravia-etal-2018-carer} with four shift environments: \textbf{(1) Group-level shift:} we partition samples into six latent groups via $K$-means clustering on prompt and response embeddings, and change the group mixture weights at test time; \textbf{(2) Prompt-level shift:} we apply synonym substitution to prompts to simulate prompt shift; \textbf{(3) Preference-level shift:} we construct two reward functions $r_1$ and $r_2$ and simulate preference shift via the convex interpolation $r_\alpha = \alpha r_1 + (1{-}\alpha)r_2$, which shifts the conditional preference distribution $P_g^S(S\mid V,X)$; and \textbf{(4) Combined shift:} all three shifts are applied simultaneously. We use Qwen2.5-0.5B-Instruct~\citep{qwen2025qwen25technicalreport} and LLaMA-3.2-1B-Instruct~\citep{grattafiori2024llama} as base models, and report average and worst-group preference accuracy~\citep{wu2025iclr-robust_drdpo} to measure overall performance and worst-case robustness, respectively. Further details on reward model training, data construction, OOD shift settings, and hyperparameter selection are provided in Appendix~\ref{app:emotion_details}.

\paragraph{Preference Poisoning Attack Robustness.} 
To evaluate robustness against training-time preference poisoning, we follow established poisoning protocols in RLHF~\citep{wang2024rlhfpoison,pathmanathan2025poisoning,kusaka2026cost} and construct poisoned training sets by flipping the preference labels (i.e., swapping the preferred and rejected responses) of a fraction of the training pairs, while keeping the test set clean. We consider two attack strategies: \textbf{(1) Random preference flip attack}~\citep{pathmanathan2025poisoning}, where poisoned pairs are selected uniformly at random; and \textbf{(2) Margin-based preference flip attack}~\citep{pathmanathan2025poisoning}, where the adversary targets the pairs with the largest margins computed by the reward model used for preference construction. We use LLaMA-3.2-1B-Instruct~\citep{grattafiori2024llama} as the base model and evaluate all methods on clean test preferences. Further details on poisoned data construction, hyperparameter selection, and evaluation settings are provided in Appendix~\ref{app:poison_attack_details}.

\paragraph{Downstream Generalization.}
To assess the downstream generalization of \method{} under potential distribution shift, we evaluate on the Open LLM Leaderboard v2~\citep{fourrier2024openllmleaderboardv2}. We fine-tune LLaMA-3.2-1B-Instruct, LLaMA-3.2-3B-Instruct, and LLaMA-3.1-8B-Instruct for one epoch. Following the pipeline of~\citet{meng2024simpo}, we collect on-policy data annotated by the ArmoRM reward model~\citep{wang-etal-2024-interpretable}. Full training and evaluation details are provided in Appendix~\ref{app:ultrafeedback_details}.

\begin{table*}[t]
\centering
\small
\caption{Performance comparison on Open LLM Leaderboard v2, which evaluates
LLMs across six tasks: MMLU-Pro, BBH, GPQA, MATH-Hard, IFEval, and MuSR.}
\label{tab:downstream_generalization}

\resizebox{0.98\linewidth}{!}{%
\begin{tabular}{l cccc cccc cccc}
\toprule
&
\multicolumn{4}{c}{LLaMA-3.2-1B} &
\multicolumn{4}{c}{LLaMA-3.2-3B} &
\multicolumn{4}{c}{LLaMA-3.1-8B} \\
\cmidrule(lr){2-5}
\cmidrule(lr){6-9}
\cmidrule(lr){10-13}

\multicolumn{1}{c}{\multirow{2}{*}{Task}}
& \multirow{2}{*}{DPO}
& \multicolumn{3}{c}{TPO}
& \multirow{2}{*}{DPO}
& \multicolumn{3}{c}{TPO}
& \multirow{2}{*}{DPO}
& \multicolumn{3}{c}{TPO} \\
\cmidrule(lr){3-5}
\cmidrule(lr){7-9}
\cmidrule(lr){11-13}

&
& $\rho_g^S=0$
& $\rho_g^S=0.01$
& $\rho_g^S=0.05$
&
& $\rho_g^S=0$
& $\rho_g^S=0.01$
& $\rho_g^S=0.05$
&
& $\rho_g^S=0$
& $\rho_g^S=0.005$
& $\rho_g^S=0.01$ \\
\midrule

MMLU-Pro
& \sC{20.0} & \sF{20.5} & \sF{20.5} & \sD{20.2}
& \sC{31.2} & \sE{31.5} & \sE{31.5} & \sF{31.7}
& \sC{36.8} & \sC{36.7} & \sF{37.1} & \sE{37.1} \\

BBH
& \sC{35.6} & \sE{35.8} & \sF{35.9} & \sC{35.6}
& \sC{45.2} & \sD{45.6} & \sF{45.9} & \sE{45.7}
& \sB{51.1} & \sC{50.9} & \sF{51.7} & \sE{51.3} \\

GPQA
& \sC{23.2} & \sE{23.5} & \sF{23.7} & \sD{23.3}
& \sC{28.0} & \sF{28.4} & \sE{28.3} & \sD{28.1}
& \sC{30.0} & \sD{30.2} & \sF{30.8} & \sD{30.1} \\

MATH-Hard
& \sC{8.1} & \sF{9.4} & \sE{8.3} & \sD{8.2}
& \sF{21.4} & \sE{21.1} & \sC{20.1} & \sD{21.0}
& \sC{19.5} & \sF{20.9} & \sE{20.7} & \sC{19.3} \\

IFEval
& \sD{51.3} & \sC{50.8} & \sE{53.0} & \sF{53.1}
& \sC{61.5} & \sD{63.5} & \sE{63.6} & \sF{65.6}
& \sC{75.4} & \sD{77.6} & \sF{79.6} & \sE{77.5} \\

MuSR
& \sF{31.6} & \sF{31.6} & \sD{31.2} & \sC{31.0}
& \sE{37.5} & \sF{37.7} & \sD{37.4} & \sC{37.1}
& \sD{38.0} & \sF{38.1} & \sD{38.0} & \sC{37.4} \\

\bottomrule
\end{tabular}%
}

\vspace{-1em}
\end{table*}

\subsection{Experimental Results}

\paragraph{OOD Preference Alignment Results.}
As shown in Table~\ref{tab:emotion_avg_worst_std_maintext}, \method{} achieves the best average accuracy across all evaluation environments on both backbones, while attaining competitive  worst-group accuracy in most settings. While the single-level variants \method{} (Group), \method{} (Prompt), and \method{} (Preference) each provide partial gains, their combination in \method{} yields stronger overall performance across both metrics and nearly all shift environments, suggesting the complementary benefits of the three robustness levels.

\paragraph{Preference Poisoning Attack Robustness Results.}
\begin{wrapfigure}{r}{0.45\linewidth}
\vspace{-1.0em}
\centering
\includegraphics[width=1.0\linewidth]{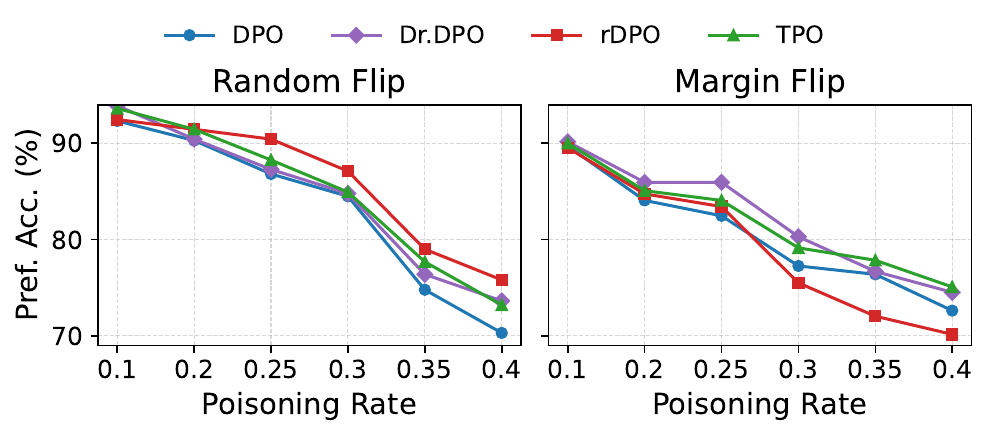}
\vspace{-2.0em}
\caption{Preference accuracy under preference poisoning.}
\label{fig:poison_attack}
\vspace{-1.0em}
\end{wrapfigure}
As shown in Figure~\ref{fig:poison_attack}, all methods degrade as the poisoning rate increases. Under random preference flip attacks, rDPO performs best as expected, given its explicit debiasing design for random label flips. \method{} is not specifically designed for this setting, yet generally outperforms DPO and remains competitive with Dr.DPO. Under margin-based flip attacks, Dr.DPO performs strongly at low poisoning rates, but \method{} surpasses all baselines at higher poisoning rates. Unlike rDPO and Dr.DPO, \method{} explicitly penalizes large positive preference margins, which may mitigate overconfidence under preference uncertainty and contribute to its stronger performance against margin-based poisoning at higher poisoning rates.

\paragraph{Downstream Generalization Results.}
As shown in Table~\ref{tab:downstream_generalization}, \method{} achieves comparable or better performance over DPO on most benchmarks across the three model scales, suggesting that the proposed robustness mechanisms largely preserve general capabilities. Gains are most pronounced on IFEval and also appear on BBH, while performance on GPQA and MMLU-Pro remains largely stable. We also observe that different preference-level robustness radii lead to varying trade-offs across tasks, which is expected given the heterogeneous nature of the leaderboard benchmarks. Interestingly, the improvements tend to become larger on IFEval as model scale increases, suggesting that larger models may benefit more from the proposed robustness mechanisms on instruction-following tasks. Overall, these results suggest that TPO can improve downstream generalization while largely retaining broad task performance.

\section{Related Work}

\paragraph{Robust Preference Optimization.} DPO simplifies preference alignment by directly optimizing over pairwise data~\citep{RafailovSMMEF23_DPO,AzarGPMRVC24}, but remains vulnerable to distribution shift at deployment~\citep{ramesh2024group_grpo,wu2025iclr-robust_drdpo}. Recent robust DPO methods address this from two angles: WDPO and KLDPO~\citep{xu2025neurips-robust_wdpo} robustify DPO using ambiguity sets over the overall preference distribution, while rDPO~\citep{chowdhury24a_rdpo} and Dr.DPO~\citep{wu2025iclr-robust_drdpo} target preference noise. Another line of work focuses on instability in preference optimization, including likelihood displacement~\citep{luo2026sharpness,razin2025unintentional,chen2024noise}, likelihood underdetermination~\citep{guo2025proximalized}, and catastrophic preference shift~\citep{jian2026_spo}. Group robust methods such as GRPO~\citep{ramesh2024group_grpo} and pluralistic alignment work~\citep{SorensenMFGMRLJCD24_pluralistic,YaoCCYJYH25_gdpo,TaoVBK24_cultural} motivate modeling preference heterogeneity across diverse user populations. However, none of these methods provide a unified robust objective that simultaneously captures group-level shift, prompt-level shift, and conditional preference shift.

\paragraph{Distributionally Robust Optimization.} DRO provides a min-max framework for learning under distribution shift~\citep{Jiao2022Distributed,jiao2025aspire,Esfahani2015DatadrivenDR,Wiesemann2014DistributionallyRC}. Wasserstein DRO~\citep{Sinha2018Certifying,Esfahani2015DatadrivenDR} admits tractable dual reformulations and connects naturally to adversarial perturbations in embedding spaces~\citep{MadryMSTV18,zhu2020freelb,jiang-etal-2020-smart}. Group DRO~\citep{SagawaKHL20_groupdro} optimizes worst-case group performance but treats each group distribution as fixed, leaving within-group shifts unaddressed~\citep{KohSMXZBHYPGLDS21,chowdhury24a_rdpo}. Recent work on hierarchical ambiguity sets~\citep{jo2025hierarchical} argues that group-level robustness alone is insufficient. Our work brings this perspective to preference optimization via a hierarchically structured rectangular ambiguity set that further decomposes intra-group uncertainty into prompt-level and preference-level shifts, thereby providing a more complete and interpretable robust alignment framework.

\section{Conclusion}
We proposed tri-level distributionally robust preference optimization (\method{}), which addresses OOD preference alignment by jointly modeling three levels of distribution shift through a hierarchically structured rectangular ambiguity set. By instantiating it with Wasserstein distances, we obtain an objective that combines worst-case group reweighting, prompt-level adversarial perturbation, and a closed-form preference-level margin penalty. Theoretical analysis establishes finite-sample convergence guarantees, and comprehensive experiments demonstrate the effectiveness of our method.

\begin{ack}

The work of Chengtao Jian and Kai Yang was supported in part by NSFC under Grant 12371519 and in part by Shenzhen Loop Area Institute under Project FP202602. The authors declare no competing interests.

\end{ack}

\bibliographystyle{plainnat}
\bibliography{neurips_2026}


\clearpage
\appendix

\section{Proofs}

\subsection{Proof of Proposition~\ref{prop:closed-form-robust-loss}}
\label{app:proof-prop1}

Recall that, for a fixed group \(g\) and a fixed prompt and response pair
\((\boldsymbol{x},\boldsymbol{y}_1,\boldsymbol{y}_2)\), the preference-level robust loss is defined as
\begin{equation}
\bar{\mathcal L}_g(\boldsymbol{x},\boldsymbol{y}_1,\boldsymbol{y}_2;\boldsymbol{\theta})
=
\sup_{q\in\mathcal U_g^S}
\mathbb E_{S\sim q}\big[
\mathcal L(\boldsymbol{x},\boldsymbol{y}_1,\boldsymbol{y}_2,S;\boldsymbol{\theta})
\big],
\end{equation}
where
\begin{equation}
\mathcal U_g^S
=
\left\{
q\in\Delta(\{+1,-1\}):
\mathcal W_{c_S}(q,\delta_{+1})\le \rho_g^S
\right\},
\qquad
c_S(s,s')=\mathbf 1[s\neq s'].
\end{equation}
Here, \(\delta_{+1}\) corresponds to the canonicalized observed preference label
\(S=+1\), obtained by ordering the candidate pair as
\((\boldsymbol{y}_1,\boldsymbol{y}_2)=(\boldsymbol{y}^+,\boldsymbol{y}^-)\).

Since the label space contains only two points, any
\(q\in\Delta(\{+1,-1\})\) can be parameterized as
$
q(+1)=1-\varepsilon, 
$
where $q(-1)=\varepsilon$ for some \(\varepsilon\in[0,1]\). Under the discrete label-flip cost \(c_S\), the
Wasserstein distance from \(q\) to \(\delta_{+1}\) is exactly the mass assigned to
the flipped label $\mathcal W_{c_S}(q,\delta_{+1})=\varepsilon$.
Therefore, assuming \(0\le \rho_g^S\le 1\), the Wasserstein constraint is equivalent to $0\le \varepsilon\le \rho_g^S$.

Let $M=m(\boldsymbol{x},\boldsymbol{y}_1,\boldsymbol{y}_2;\boldsymbol{\theta})$. Using the labeled DPO loss $\mathcal L(\boldsymbol{x},\boldsymbol{y}_1,\boldsymbol{y}_2,S;\boldsymbol{\theta})=\varphi(S\cdot M)$, we have
\begin{equation}
\begin{aligned}
\bar{\mathcal L}_g(\boldsymbol{x},\boldsymbol{y}_1,\boldsymbol{y}_2;\boldsymbol{\theta})
&=
\sup_{0\le \varepsilon\le \rho_g^S}
\Big[
(1-\varepsilon)\varphi(M)+\varepsilon\,\varphi(-M)
\Big] \\
&=
\varphi(M)
+
\sup_{0\le \varepsilon\le \rho_g^S}
\varepsilon\bigl(\varphi(-M)-\varphi(M)\bigr).
\end{aligned}
\end{equation}
For \(\varphi(t)=\log(1+e^{-t})\), we have $\varphi(-t)-\varphi(t)=t$. Thus, we can rewrite the preference-level robust loss as
\begin{equation}
\bar{\mathcal L}_g(\boldsymbol{x},\boldsymbol{y}_1,\boldsymbol{y}_2;\boldsymbol{\theta})
=
\varphi(M)
+
\sup_{0\le \varepsilon\le \rho_g^S}
\varepsilon M.
\end{equation}
Since the last term is linear in \(\varepsilon\), its supremum is attained at
\(\varepsilon=\rho_g^S\) when \(M\ge 0\), and at \(\varepsilon=0\) when \(M<0\).
Therefore, we have
\begin{equation}
\sup_{0\le \varepsilon\le \rho_g^S}
\varepsilon M=\rho_g^S[M]_+.
\end{equation}
Substituting this back, we obtain
\begin{equation}
\bar{\mathcal L}_g(\boldsymbol{x},\boldsymbol{y}_1,\boldsymbol{y}_2;\boldsymbol{\theta})
=
\varphi\bigl(m(\boldsymbol{x},\boldsymbol{y}_1,\boldsymbol{y}_2;\boldsymbol{\theta})\bigr)
+
\rho_g^S
\bigl[
m(\boldsymbol{x},\boldsymbol{y}_1,\boldsymbol{y}_2;\boldsymbol{\theta})
\bigr]_+,
\end{equation}
which completes the proof of Proposition~\ref{prop:closed-form-robust-loss}.

\subsection{Proof of Proposition~\ref{prop:optimal-margin}}
\label{app:proof-prop2}

Let $\mathcal L_g(M)=\varphi(M)+\rho_g^S[M]_+$ and $\varphi(M)=\log(1+e^{-M})$. We analyze the minimizer over \(M\in\mathbb R\). First, when \(\rho_g^S=0\), the objective reduces to $\mathcal L_g(M)=\varphi(M)$. Since $\varphi'(M)=-\frac{1}{1+e^M}<0$, \(\varphi(M)\) is strictly decreasing on \(\mathbb R\).  Therefore, its infimum is approached as \(M\to+\infty\), and the optimal target margin is \(M_g^\star=+\infty\).

We now consider the case \(\rho_g^S>0\). Since \([M]_+\) is piecewise linear, we
analyze the cases \(M<0\) and \(M>0\) separately. For \(M<0\), we have \([M]_+=0\), hence, we obtain $\mathcal L_g(M)=\varphi(M)$.  Its derivative satisfies $\mathcal L_g'(M)=\varphi'(M)=-\frac{1}{1+e^{M}}<0$. Therefore, \(\mathcal L_g(M)\) is strictly decreasing on \((-\infty,0)\). 

For \(M>0\), we have \([M]_+=M\), and thus
\begin{equation}
\mathcal L_g(M)=\varphi(M)+\rho_g^S M.
\end{equation}
Its derivative is
\begin{equation}
\mathcal L_g'(M)
=
-\frac{1}{1+e^{M}}+\rho_g^S.
\end{equation}
Setting \(\mathcal L_g'(M)=0\) yields
\begin{equation}
\frac{1}{1+e^{M}}=\rho_g^S
\quad\Longleftrightarrow\quad
e^{M}=\frac{1-\rho_g^S}{\rho_g^S},
\end{equation}
so any stationary point on \((0,\infty)\) must satisfy
\begin{equation}
M=\log\frac{1-\rho_g^S}{\rho_g^S}.
\end{equation}
This stationary point is positive if and only if
\(0<\rho_g^S<\tfrac12\). 
When \(0<\rho_g^S<\tfrac12\), the derivative
\(\mathcal L_g'(M)\) changes sign from negative to positive at
\begin{equation}
M_g^\star=\log\frac{1-\rho_g^S}{\rho_g^S},
\end{equation}
so this point is the global minimizer.
When \(\rho_g^S\ge\tfrac12\), the stationary point above is non-positive.
Moreover, for all \(M>0\), we have
\begin{equation}
\mathcal L_g'(M)
=
-\frac{1}{1+e^{M}}+\rho_g^S
>0,
\end{equation}
so \(\mathcal L_g(M)\) is increasing on \((0,\infty)\). Since it is decreasing on
\((-\infty,0)\), the global minimizer is attained at the boundary \(M=0\).

Therefore, we have shown that the optimal target margin is
\begin{equation}
M_g^\star
\in
\arg\min_{M\in\mathbb R}\mathcal L_g(M)
=
\begin{cases}
+\infty, & \rho_g^S=0,\\[4pt]
\log\dfrac{1-\rho_g^S}{\rho_g^S}, & 0<\rho_g^S<\frac12,\\[8pt]
0, & \rho_g^S\ge \frac12,
\end{cases}
\end{equation}
which proves Proposition~\ref{prop:optimal-margin}.

\subsection{Proof of Theorem~\ref{thm:hpo-main}}
\label{app:proof-theorem1}

Theorem~\ref{thm:hpo-main} analyzes the population and empirical
minimizers of the fixed-penalty Wasserstein dual objective, with each
penalty parameter \(\gamma_g>0\) held fixed. Throughout this subsection,
\(\mathcal R_g\) and \(\mathcal R_{\rm \method{}}\) refer to this
fixed-penalty objective, without minimization over \(\gamma_g\).
For a sample, the DPO margin is defined as
\begin{equation}
m(\boldsymbol h,\boldsymbol y^+,\boldsymbol y^-;\boldsymbol \theta)
=
\beta
\left(
\log\frac{\pi_{\boldsymbol \theta}(\boldsymbol y^+\mid \boldsymbol h)}{\pi_{\rm ref}(\boldsymbol y^+\mid \boldsymbol h)}
-
\log\frac{\pi_{\boldsymbol \theta}(\boldsymbol y^-\mid \boldsymbol h)}{\pi_{\rm ref}(\boldsymbol y^-\mid \boldsymbol h)}
\right).
\end{equation}
The logistic loss is denoted by \(\varphi(t)=\log(1+e^{-t})\). The
preference-level robust loss for group \(g\) is
\begin{equation}
\bar{\mathcal L}_g(\boldsymbol h,\boldsymbol y^+,\boldsymbol y^-;\boldsymbol \theta)
=
\varphi(m(\boldsymbol h,\boldsymbol y^+,\boldsymbol y^-;\boldsymbol \theta))
+
\rho_g^S[m(\boldsymbol h,\boldsymbol y^+,\boldsymbol y^-;\boldsymbol \theta)]_+.
\end{equation}
The fixed-penalty local robust sample loss is
\begin{equation}
\ell_{\boldsymbol\theta}^g(U;\gamma_g)
=
\sup_{\boldsymbol \delta\in\mathbb R^d}
\left[
\bar{\mathcal L}_g(\boldsymbol h+\boldsymbol \delta,\boldsymbol y^+,\boldsymbol y^-;\boldsymbol \theta)
-
\frac{\gamma_g}{2}\|\boldsymbol \delta\|_2^2
\right].
\end{equation}
The population and empirical local robust risk for group \(g\) are
\begin{equation}
\mathcal R_g(\boldsymbol\theta)
=
\gamma_g\rho_g^X
+
\mathbb E_{P_g}\big[\ell_{\boldsymbol\theta}^g(U;\gamma_g)\big],
\quad
\widehat{\mathcal R}_g(\boldsymbol\theta)
=
\gamma_g\rho_g^X
+
\frac1{n_g}
\sum_{i=1}^{n_g}
\ell_{\boldsymbol\theta}^g(U_{g,i};\gamma_g).
\end{equation}
The hierarchical population and empirical \method{} objectives are defined as
\begin{equation}
\mathcal R_{\rm \method{}}(\boldsymbol\theta)
=
\sup_{\boldsymbol\omega\in\mathcal U^\omega}
\sum_{g=1}^m
\omega_g\mathcal R_g(\boldsymbol\theta),
\quad
\widehat{\mathcal R}_{\rm \method{}}(\boldsymbol\theta)
=
\sup_{\boldsymbol\omega\in\mathcal U^\omega}
\sum_{g=1}^m
\omega_g\widehat{\mathcal R}_g(\boldsymbol\theta).
\end{equation}
Throughout the proof, we assume that samples are independent within each group, \(U_{g,1},\ldots,U_{g,n_g}\overset{\rm i.i.d.}{\sim}P_g\).

\begin{lemma}[Boundedness and Lipschitzness of the DPO margin]
\label{lem:margin-bounded-lipschitz}
Suppose Assumption~\ref{ass:loglinear-bounded} holds. Then, for each original
or adversarially perturbed prompt embedding \(\tilde{\boldsymbol{h}}\), each response pair \((\boldsymbol y^+,\boldsymbol y^-)\), and each \(\boldsymbol\theta,\boldsymbol\theta'\in\Theta\),
\begin{equation}
|m(\tilde {\boldsymbol{h}},\boldsymbol y^+,\boldsymbol y^-;\boldsymbol\theta)|
\le
M_\beta,
\qquad
M_\beta:=4\beta B C_\phi,
\end{equation}
and
\begin{equation}
|m(\tilde {\boldsymbol{h}},\boldsymbol y^+,\boldsymbol y^-;\boldsymbol\theta)
-
m(\tilde {\boldsymbol{h}},\boldsymbol y^+,\boldsymbol y^-;\boldsymbol\theta')|
\le
2\beta C_\phi
\|\boldsymbol\theta-\boldsymbol\theta'\|_2.
\end{equation}
\end{lemma}

\textit{Proof.}
Under the log-linear policy class, we have
\begin{equation}
\log \pi_{\boldsymbol\theta}(\boldsymbol y^+\mid \tilde {\boldsymbol{h}})
-
\log \pi_{\boldsymbol\theta}(\boldsymbol y^-\mid \tilde {\boldsymbol{h}})
=
\boldsymbol\theta^\top
\big(
\phi(\tilde {\boldsymbol{h}},\boldsymbol y^+)-\phi(\tilde {\boldsymbol{h}},\boldsymbol y^-)
\big).
\end{equation}
Since \(\pi_{\rm ref}\) belongs to the same class with parameter \(\boldsymbol\theta_{\rm ref}\), we have
\begin{equation}
m(\tilde {\boldsymbol{h}},\boldsymbol y^+,\boldsymbol y^-;\boldsymbol\theta)
=
\beta
\left\langle
\phi(\tilde {\boldsymbol{h}},\boldsymbol y^+)-\phi(\tilde {\boldsymbol{h}},\boldsymbol y^-),
\boldsymbol\theta-\boldsymbol\theta_{\rm ref}
\right\rangle.
\end{equation}
By Assumption~\ref{ass:loglinear-bounded}, we have
\begin{equation}
\|\phi(\tilde {\boldsymbol{h}},\boldsymbol y^+)-\phi(\tilde {\boldsymbol{h}},\boldsymbol y^-)\|_2\le 2C_\phi,
\qquad
\|\boldsymbol\theta-\boldsymbol\theta_{\rm ref}\|_2\le 2B.
\end{equation}
Therefore,
\begin{equation}
|m(\tilde {\boldsymbol{h}},\boldsymbol y^+,\boldsymbol y^-;\boldsymbol\theta)|
\le
\beta\cdot 2C_\phi\cdot 2B
=
4\beta B C_\phi
=
M_\beta.
\end{equation}
So we can obtain
\begin{equation}
\begin{aligned}
&|m(\tilde {\boldsymbol{h}},\boldsymbol y^+,\boldsymbol y^-;\boldsymbol\theta)
-
m(\tilde {\boldsymbol{h}},\boldsymbol y^+,\boldsymbol y^-;\boldsymbol\theta')|\\
& =
\beta
\left|
\left\langle
\phi(\tilde {\boldsymbol{h}},\boldsymbol y^+)-\phi(\tilde {\boldsymbol{h}},\boldsymbol y^-),
\boldsymbol\theta-\boldsymbol\theta'
\right\rangle
\right|\\
& \le
\beta
\|\phi(\tilde {\boldsymbol{h}},\boldsymbol y^+)-\phi(\tilde {\boldsymbol{h}},\boldsymbol y^-)\|_2
\|\boldsymbol\theta-\boldsymbol\theta'\|_2\\
&\le
2\beta C_\phi
\|\boldsymbol\theta-\boldsymbol\theta'\|_2,
\end{aligned}
\end{equation}
which completes the proof.

\begin{lemma}[Boundedness and Lipschitzness of $\bar{\mathcal L}_g$]
\label{lem:robust-loss-bounded-lipschitz}
Suppose Assumption~\ref{ass:loglinear-bounded} holds. Then, for each group \(g\), each original or adversarially perturbed prompt embedding \(\tilde{\boldsymbol{h}}\), each response pair \((\boldsymbol y^+,\boldsymbol y^-)\), and each \(\boldsymbol\theta,\boldsymbol\theta'\in\Theta\), we have
\[
0\le
\bar{\mathcal L}_g(\tilde {\boldsymbol{h}},\boldsymbol y^+,\boldsymbol y^-;\boldsymbol\theta)
\le
K_g,
\]
where $K_g=\log(1+e^{M_\beta})+\rho_g^S M_\beta$ and $M_\beta=4\beta B C_\phi$.
Moreover, $\bar{\mathcal L}_g(\tilde {\boldsymbol{h}},\boldsymbol y^+,\boldsymbol y^-;\boldsymbol\theta)$ is $L_g$-Lipschitz, and $L_g=2\beta C_\phi(1+\rho_g^S)$.
\end{lemma}

\textit{Proof. }
By definition,
\begin{equation}
\bar{\mathcal L}_g(\tilde {\boldsymbol{h}},\boldsymbol y^+,\boldsymbol y^-;\theta)
=
\varphi(m(\tilde {\boldsymbol{h}},\boldsymbol y^+,\boldsymbol y^-;\theta))
+
\rho_g^S[m(\tilde {\boldsymbol{h}},\boldsymbol y^+,\boldsymbol y^-;\theta)]_+.
\end{equation}
Since \(\varphi(t)\ge0\), \([t]_+\ge0\), and \(\rho_g^S\ge0\), we have $\bar{\mathcal L}_g(\tilde {\boldsymbol{h}},\boldsymbol y^+,\boldsymbol y^-;\boldsymbol\theta)\ge0$.

According to Lemma~\ref{lem:margin-bounded-lipschitz}, we have $|m(\tilde {\boldsymbol{h}},\boldsymbol y^+,\boldsymbol y^-;\boldsymbol\theta)|\le M_\beta$. Hence, we obtain
\begin{align}
\varphi\!\left(m(\tilde{\boldsymbol{h}},\boldsymbol y^+,\boldsymbol y^-;\boldsymbol\theta)\right)
&\le
\log(1+e^{M_\beta}),
\label{eq:logistic-bound}\\
\left[m(\tilde{\boldsymbol{h}},\boldsymbol y^+,\boldsymbol y^-;\boldsymbol\theta)\right]_+
&\le
M_\beta.
\label{eq:positive-part-bound}
\end{align}
Combining \eqref{eq:logistic-bound} and \eqref{eq:positive-part-bound}, we get
\begin{equation}
\bar{\mathcal L}_g(\tilde {\boldsymbol{h}},\boldsymbol y^+,\boldsymbol y^-;\boldsymbol\theta)
\le
\log(1+e^{M_\beta})
+
\rho_g^S M_\beta
=
K_g.
\end{equation}

For Lipschitzness, \(\varphi(t)=\log(1+e^{-t})\) is \(1\)-Lipschitz because $|\varphi'(t)|=\frac{1}{1+e^t}\le1$. And \(t\mapsto[t]_+\) is also \(1\)-Lipschitz. Therefore,
\(f(t)= \varphi(t)+\rho_g^S[t]_+\) is
\((1+\rho_g^S)\)-Lipschitz. Combining this with
Lemma~\ref{lem:margin-bounded-lipschitz}, we have
\begin{equation}
\begin{aligned}
&\left|
\bar{\mathcal L}_g(\tilde {\boldsymbol{h}},\boldsymbol y^+,\boldsymbol y^-;\boldsymbol\theta)
-
\bar{\mathcal L}_g(\tilde {\boldsymbol{h}},\boldsymbol y^+,\boldsymbol y^-;\boldsymbol\theta')
\right|\\
& \le
(1+\rho_g^S)
\left|
m(\tilde {\boldsymbol{h}},\boldsymbol y^+,\boldsymbol y^-;\boldsymbol\theta)
-
m(\tilde {\boldsymbol{h}},\boldsymbol y^+,\boldsymbol y^-;\boldsymbol\theta')
\right|\\
& \le
2\beta C_\phi(1+\rho_g^S)
\|\boldsymbol\theta-\boldsymbol\theta'\|_2, 
\end{aligned}
\end{equation}
which completes the proof.

\begin{lemma}[Rademacher complexity of the local robust loss class]
\label{lem:local-robust-rademacher}
Let $\mathcal F_g=\{\ell_{\boldsymbol\theta}^g(\cdot;\gamma_g):\boldsymbol\theta\in\Theta\}$, where \(\gamma_g>0\) is fixed. Under Assumption~\ref{ass:loglinear-bounded},
the class \(\mathcal F_g\) satisfies
\[
\mathfrak R_{n_g}(\mathcal F_g)
\le
A_g\sqrt{\frac{d}{n_g}},
\]
where $A_g=C K_g \sqrt{ \log\left( e+\frac{3BL_g}{K_g} \right) } $, \(C>0\) is a universal numerical constant, $K_g=\log(1+e^{M_\beta})+\rho_g^S M_\beta$, $M_\beta=4\beta B C_\phi$, and $L_g=2\beta C_\phi(1+\rho_g^S).$
\end{lemma}

\textit{Proof. }
We first prove that the fixed-penalty local robust loss
\(\ell_{\boldsymbol\theta}^g(U;\gamma_g)\) is bounded and Lipschitz in \(\boldsymbol\theta\).
By Lemma~\ref{lem:robust-loss-bounded-lipschitz}, for every admissible
perturbation \(\boldsymbol\delta\), we have $0\le\bar{\mathcal L}_g(\boldsymbol h+\boldsymbol\delta,\boldsymbol y^+,\boldsymbol y^-;\boldsymbol\theta)\le K_g$.

Since \(-\frac{\gamma_g}{2}\|\boldsymbol\delta\|^2\le0\), we can obtain
\begin{equation}
\ell_{\boldsymbol\theta}^g(U;\gamma_g)
=
\sup_{\boldsymbol\delta}
\left[
\bar{\mathcal L}_g(\boldsymbol h+\boldsymbol\delta,\boldsymbol y^+,\boldsymbol y^-;\boldsymbol\theta)
-
\frac{\gamma_g}{2}\|\boldsymbol\delta\|_2^2
\right]
\le
K_g.
\end{equation}
Moreover, because \(\boldsymbol\delta=0\) is admissible, $\ell_{\boldsymbol\theta}^g(U;\gamma_g) \ge \bar{\mathcal L}_g(\boldsymbol h,\boldsymbol y^+,\boldsymbol y^-;\boldsymbol\theta)\ge 0$. Thus, we have \(0\le\ell_{\boldsymbol\theta}^g(U;\gamma_g)\le K_g\).

For Lipschitzness, we define
\begin{equation}
\mathcal J_{\boldsymbol{\theta}}(\boldsymbol{\delta})
=
\bar{\mathcal L}_g(\boldsymbol h+\boldsymbol\delta,\boldsymbol y^+,\boldsymbol y^-;\boldsymbol\theta)
-
\frac{\gamma_g}{2}\|\boldsymbol\delta\|_2^2.
\end{equation}
Then, we have
\begin{equation}
\begin{aligned}
|\ell_{\boldsymbol\theta}^g(U;\gamma_g)-\ell_{\boldsymbol\theta'}^g(U;\gamma_g)|
&=
\left|
\sup_{\boldsymbol\delta} \mathcal J_{\boldsymbol{\theta}}(\boldsymbol{\delta})
-
\sup_{\boldsymbol\delta} \mathcal J_{\boldsymbol{\theta'}}(\boldsymbol{\delta})
\right|\\
&\le
\sup_{\boldsymbol\delta}
|\mathcal J_{\boldsymbol{\theta}}(\boldsymbol{\delta})-\mathcal J_{\boldsymbol{\theta'}}(\boldsymbol{\delta})|\\
&=
\sup_{\boldsymbol\delta}
\left|
\bar{\mathcal L}_g(\boldsymbol h+\boldsymbol\delta,\boldsymbol y^+,\boldsymbol y^-;\boldsymbol\theta)
-
\bar{\mathcal L}_g(\boldsymbol h+\boldsymbol\delta,\boldsymbol y^+,\boldsymbol y^-;\boldsymbol\theta')
\right|\\
&\le
L_g\|\boldsymbol\theta-\boldsymbol\theta'\|_2.
\end{aligned}
\end{equation}
Hence, $\ell_{\boldsymbol\theta}^g$ is \(L_g\)-Lipschitz in \(\boldsymbol\theta\) with \(L_g=2\beta C_\phi(1+\rho_g^S)\).

Thus, an \(\frac{\varepsilon}{L_g}\)-cover of \(\Theta\) under
\(\|\cdot\|_2\) induces an \(\varepsilon\)-cover of \(\mathcal F_g\) under
\(\|\cdot\|_\infty\). Indeed, for any \(\boldsymbol\theta\in\Theta\), choose
\(\boldsymbol\theta_j\) such that
\(\|\boldsymbol\theta-\boldsymbol\theta_j\|_2\le \frac{\varepsilon}{L_g}\). Then
\begin{equation}
\|\ell_{\boldsymbol\theta}^g(\cdot;\gamma_g)
-
\ell_{\boldsymbol\theta_j}^g(\cdot;\gamma_g)\|_\infty
\le
L_g\|\boldsymbol\theta-\boldsymbol\theta_j\|_2
\le
\varepsilon.
\end{equation}
Hence, we can get
\begin{equation}
N(\varepsilon,\mathcal F_g,\|\cdot\|_\infty)
\le
N(\frac{\varepsilon}{L_g},\Theta,\|\cdot\|_2).
\end{equation}

By Dudley's entropy integral for Rademacher complexity~\citep{Wainwright2019}, the entropy bound above gives
\begin{equation}
\label{eq:local-entropy-bound}
\mathfrak R_{n_g}(\mathcal F_g)
\le
\frac{C}{\sqrt{n_g}}
\int_0^{K_g}
\sqrt{
\log N(\varepsilon,\mathcal F_g,\|\cdot\|_\infty)
}
\,d\varepsilon .
\end{equation}

By the standard volumetric covering bound~\citep{Vershynin2018,Wainwright2019}, we have that
\begin{equation}
\log N(\varepsilon,\mathcal F_g,\|\cdot\|_\infty)
\le
d\log\left(1+\frac{3BL_g}{\varepsilon}\right) 
\le d\log\left(e+\frac{3BL_g}{\varepsilon}\right) 
\end{equation}
So we obtain 
\begin{equation}
\label{eq:local-dudley-bound}
\mathfrak R_{n_g}(\mathcal F_g)
\le
\frac{C}{\sqrt{n_g}}
\int_0^{K_g}
\sqrt{
d\log\left(e+\frac{3BL_g}{\varepsilon}\right)
}
\,d\varepsilon .
\end{equation}

Let \(R_g=3BL_g\). It remains to bound the following integral:
\begin{equation}
I_g=\int_0^{K_g}\sqrt{\log\left(e+\frac{R_g}{\varepsilon}\right)}\,d\varepsilon .
\end{equation}
By setting \(u=\frac{\varepsilon}{K_g}\) and \(a_g=\frac{R_g}{K_g}\), we can rewrite this integral as
\begin{equation}
I_g
=
K_g\int_0^1\sqrt{\log\left(e+\frac{a_g}{u}\right)}\,du.
\end{equation}
For \(u\in(0,1]\), we have \(e+\frac{a_g}{u}\le \frac{e+a_g}{u}\), and hence
\begin{equation}
\log\left(e+\frac{a_g}{u}\right)
\le
\log(e+a_g)+\log(1/u).
\end{equation}
Using \(\sqrt{x+y}\le \sqrt{x}+\sqrt{y}\), we obtain
\begin{equation}
I_g
\le
K_g\sqrt{\log(e+a_g)}
+
K_g\int_0^1\sqrt{\log(1/u)}\,du.
\end{equation}
Moreover,
\begin{equation}
\int_0^1\sqrt{\log(1/u)}\,du
=
\int_0^\infty s^{1/2}e^{-s}\,ds
=
\Gamma(3/2)
=
\frac{\sqrt{\pi}}{2}.
\end{equation}
Since \(\log(e+a_g)\ge 1\), the fixed numerical term
\(\frac{\sqrt{\pi}}{2}\) can be absorbed into a universal constant. Therefore,
\begin{equation}
I_g
\le
C K_g\sqrt{\log(e+a_g)}
=
C K_g
\sqrt{
\log\left(e+\frac{3BL_g}{K_g}\right)
}.
\end{equation}
Substituting this into Eq.~\eqref{eq:local-dudley-bound}, we obtain
\begin{equation}
\mathfrak R_{n_g}(\mathcal F_g)
\le
C K_g
\sqrt{
\log\left(e+\frac{3BL_g}{K_g}\right)
}
\sqrt{\frac{d}{n_g}}.
\end{equation}
By the definition of $A_g=C K_g\sqrt{\log\left(e+\frac{3BL_g}{K_g}\right)}$, we obtain
\begin{equation}
\mathfrak R_{n_g}(\mathcal F_g)
\le
A_g\sqrt{\frac{d}{n_g}}.
\end{equation}
This completes the proof of Lemma~\ref{lem:local-robust-rademacher}.

\begin{lemma}[Group-wise local uniform convergence]
\label{lem:local-uniform-convergence}
Suppose Assumption~\ref{ass:loglinear-bounded} holds. Then, with probability at least \(1-\delta\), for
all \(g\in[m]\), we have
\[
\sup_{\boldsymbol\theta\in\Theta}
\left|
\mathcal R_g(\boldsymbol\theta)
-
\widehat{\mathcal R}_g(\boldsymbol\theta)
\right|
\le
\epsilon_g,
\]
where $\epsilon_g=2A_g\sqrt{\frac{d}{n_g}}+K_g\sqrt{\frac{\log(2m/\delta)}{2n_g}}$.
\end{lemma}

\textit{Proof. }
For a fixed group \(g\), the constant term \(\gamma_g\rho_g^X\) appears in
both \(\mathcal R_g(\boldsymbol\theta)\) and \(\widehat{\mathcal R}_g(\boldsymbol\theta)\), so it
cancels and we have
\begin{equation}
\mathcal R_g(\boldsymbol\theta)-\widehat{\mathcal R}_g(\boldsymbol\theta)
=
\mathbb E_{P_g}[\ell_{\boldsymbol\theta}^g(U;\gamma_g)]
-
\frac1{n_g}\sum_{i=1}^{n_g}\ell_{\boldsymbol\theta}^g(U_{g,i};\gamma_g).
\end{equation}
By the standard Rademacher uniform convergence bound~\citep{BartlettMendelson2002} for functions bounded in
\([0,K_g]\), with probability at least \(1-\delta_g\),
\begin{equation}
\sup_{\boldsymbol\theta\in\Theta}
\left|
\mathcal R_g(\boldsymbol\theta)-\widehat{\mathcal R}_g(\boldsymbol\theta)
\right|
\le
2\mathfrak R_{n_g}(\mathcal F_g)
+
K_g
\sqrt{
\frac{\log(2/\delta_g)}{2n_g}
}.
\end{equation}
Using Lemma~\ref{lem:local-robust-rademacher},
\begin{equation}
\mathfrak R_{n_g}(\mathcal F_g)
\le
A_g\sqrt{\frac{d}{n_g}}.
\end{equation}
Setting \(\delta_g=\frac{\delta}{m}\) and applying a union bound over
\(g=1,\ldots,m\) gives the result.

\begin{lemma}[From local convergence to hierarchical convergence]
\label{lem:hierarchical-uniform-convergence}
Suppose the event in Lemma~\ref{lem:local-uniform-convergence} holds and
\(\mathcal U^\omega\subseteq\Delta^{m-1}\). Then
\[
\sup_{\boldsymbol\theta \in \Theta}
\left|
\mathcal R_{\rm TPO}(\boldsymbol\theta)
-
\widehat{\mathcal R}_{\rm TPO}(\boldsymbol\theta)
\right|
\le
\epsilon_{\rm H},
\]
where $\epsilon_{\rm H}=\sup_{\boldsymbol\omega\in\mathcal U^\omega}\sum_{g=1}^m\omega_g\epsilon_g$. In particular, $\epsilon_{\rm H}\le\max_{g\in[m]}\epsilon_g$. 
\end{lemma}

\textit{Proof. }
Fix \(\boldsymbol\theta\in\Theta\). Using
\begin{equation}
\left|
\sup_x \mathcal A_{\boldsymbol{\theta}}(x)-\sup_x \mathcal A_{\boldsymbol{\theta'}}(x)
\right|
\le
\sup_x |\mathcal A_{\boldsymbol{\theta}}(x)-\mathcal A_{\boldsymbol{\theta'}}(x)|,
\end{equation}
we obtain
\begin{equation}
\begin{aligned}
&
\left|
\mathcal R_{\rm \method{}}(\boldsymbol\theta)
-
\widehat{\mathcal R}_{\rm \method{}}(\boldsymbol\theta)
\right|\\
&=
\left|
\sup_{\omega\in\mathcal U^\omega}
\sum_{g=1}^m\omega_g\mathcal R_g(\boldsymbol\theta)
-
\sup_{\omega\in\mathcal U^\omega}
\sum_{g=1}^m\omega_g\widehat{\mathcal R}_g(\boldsymbol\theta)
\right|\\
&\le
\sup_{\omega\in\mathcal U^\omega}
\left|
\sum_{g=1}^m
\omega_g
\left(
\mathcal R_g(\boldsymbol\theta)
-
\widehat{\mathcal R}_g(\boldsymbol\theta)
\right)
\right|\\
&\le
\sup_{\omega\in\mathcal U^\omega}
\sum_{g=1}^m
\omega_g
\left|
\mathcal R_g(\boldsymbol\theta)
-
\widehat{\mathcal R}_g(\boldsymbol\theta)
\right|.
\end{aligned}
\end{equation}
By Lemma~\ref{lem:local-uniform-convergence}, for all $g\in[m]$ and $\boldsymbol\theta\in\Theta$, we have
\begin{equation}
\left|
\mathcal R_g(\boldsymbol\theta)
-
\widehat{\mathcal R}_g(\boldsymbol\theta)
\right|
\le
\epsilon_g
\end{equation}
Therefore, we can get 
\begin{equation}
\left|
\mathcal R_{\rm \method{}}(\boldsymbol\theta)
-
\widehat{\mathcal R}_{\rm \method{}}(\boldsymbol\theta)
\right|
\le
\sup_{\omega\in\mathcal U^\omega}
\sum_{g=1}^m
\omega_g\epsilon_g
=
\epsilon_{\rm H}.
\end{equation}
Since \(\mathcal U^\omega\subseteq\Delta^{m-1}\), we have
\(\omega_g\ge0\) and \(\sum_g\omega_g=1\). Hence,
\begin{equation}
\sup_{\omega\in\mathcal U^\omega}
\sum_{g=1}^m
\omega_g\epsilon_g
\le
\max_{g\in[m]}\epsilon_g.
\end{equation}

\begin{lemma}[Strong convexity of the population \method{} objective]
\label{lem:strong-convexity-coverage}
Suppose Assumptions~\ref{ass:loglinear-bounded} and
\ref{ass:uniform-coverage} hold. Let $M_\beta=4\beta B C_\phi$ and $\kappa_\beta=\frac{\beta^2 e^{M_\beta}}{(1+e^{M_\beta})^2}$. Then the population risk \(\mathcal R_{\rm TPO}\) is
\(\kappa_\beta\lambda\)-strongly convex over \(\Theta\). Consequently, for $\boldsymbol\theta^\star\in\arg\min_{\boldsymbol\theta\in\Theta}\mathcal R_{TPO}(\boldsymbol\theta)$, we have
\[
\mathcal R_{\rm TPO}(\boldsymbol\theta)
-
\mathcal R_{\rm TPO}(\boldsymbol\theta^\star)
\ge
\frac{\kappa_\beta\lambda}{2}
\|\boldsymbol\theta-\boldsymbol\theta^\star\|_2^2
\]
\end{lemma}

\textit{Proof. }
For a fixed sample \(U=(\boldsymbol h,\boldsymbol y^+,\boldsymbol y^-)\), let $z(U)=\phi(\boldsymbol h,\boldsymbol y^+)-\phi(\boldsymbol h,\boldsymbol y^-)$. Under the log-linear policy class, we have
\begin{equation}
m(\boldsymbol h,\boldsymbol y^+,\boldsymbol y^-;\boldsymbol\theta)
=
\beta\langle z(U),\boldsymbol\theta-\boldsymbol\theta_{\rm ref}\rangle.
\end{equation}
The Hessian of the DPO logistic loss \(\varphi(m(\boldsymbol h,\boldsymbol y^+,\boldsymbol y^-;\boldsymbol\theta))\) with
respect to \(\boldsymbol\theta\) is
\begin{equation}
\nabla_{\boldsymbol\theta}^2 \varphi(m(\boldsymbol h,\boldsymbol y^+,\boldsymbol y^-;\boldsymbol\theta))
=
\beta^2
\sigma(m(\boldsymbol h,\boldsymbol y^+,\boldsymbol y^-;\boldsymbol\theta))
\sigma(-m(\boldsymbol h,\boldsymbol y^+,\boldsymbol y^-;\boldsymbol\theta))
z(U)z(U)^\top.
\end{equation}
By Lemma~\ref{lem:margin-bounded-lipschitz}, we have
\begin{equation}
|m(\boldsymbol h, \boldsymbol y^+,\boldsymbol y^-;\boldsymbol\theta)|\le M_\beta.
\end{equation}
Since
\begin{equation}
\sigma(t)\sigma(-t)=\frac{e^t}{(1+e^t)^2},
\end{equation}
for all \(|t|\le M_\beta\), we can obtain
\begin{equation}
\sigma(t)\sigma(-t)
\ge
\frac{e^{M_\beta}}{(1+e^{M_\beta})^2}.
\end{equation}
Therefore, we have
\begin{equation}
\nabla_{\boldsymbol\theta} ^2 \varphi(m(\boldsymbol h,\boldsymbol y^+,\boldsymbol y^-;\boldsymbol \theta))
\succeq
\kappa_\beta z(U)z(U)^\top.
\end{equation}

We now establish strong convexity of the hierarchical
fixed-penalty objective. For $\boldsymbol\theta_1,\boldsymbol\theta_2
\in\Theta$ and $t\in[0,1]$, let
$
\boldsymbol\theta_t
=
t\boldsymbol\theta_1+
(1-t)\boldsymbol\theta_2.
$
Choose worst-case group weights
$\boldsymbol\omega_t\in\mathcal U^\omega$
and measurable optimal perturbations
$T_{g,\boldsymbol\theta_t}(U)$
at $\boldsymbol\theta_t$.

We define
\begin{equation}
\begin{aligned}
F_t(\boldsymbol\theta)
=
\sum_{g=1}^m\omega_{t,g}
\Bigg[
&\gamma_g\rho_g^X
+\mathbb E_{P_g}
\Big[
\bar{\mathcal L}_g(
\boldsymbol h+T_{g,\boldsymbol\theta_t}(U),
\boldsymbol y^+,\boldsymbol y^-;
\boldsymbol\theta)
-
\frac{\gamma_g}{2}
\|T_{g,\boldsymbol\theta_t}(U)\|_2^2
\Big]
\Bigg].
\end{aligned}
\end{equation}

By the preceding Hessian bound and
Assumption~\ref{ass:uniform-coverage},
the logistic-loss component of $F_t$
has curvature at least $\kappa_\beta\lambda$.
The preference-level penalty
$\rho_g^S[m]_+$ is convex in
$\boldsymbol\theta$, and the perturbation
penalty is independent of $\boldsymbol\theta$.
Thus, $F_t$ is
$\kappa_\beta\lambda$-strongly convex.

By optimality of the selected weights
and perturbations,
we have
$
F_t(\boldsymbol\theta_t)
=
\mathcal R_{\rm \method{}}
(\boldsymbol\theta_t),
$
and 
\[
F_t(\boldsymbol\theta)
\le
\mathcal R_{\rm \method{}}
(\boldsymbol\theta)
\quad\text{for all }\boldsymbol\theta\in\Theta.
\]
Consequently,
\begin{equation}
\begin{aligned}
\mathcal R_{\rm \method{}}
(\boldsymbol\theta_t)
&\le
t\mathcal R_{\rm \method{}}
(\boldsymbol\theta_1)
+
(1-t)\mathcal R_{\rm \method{}}
(\boldsymbol\theta_2)\\
&\quad-
\frac{\kappa_\beta\lambda}{2}
t(1-t)
\|\boldsymbol\theta_1-\boldsymbol\theta_2\|_2^2.
\end{aligned}
\end{equation}

Hence, $\mathcal R_{\rm \method{}}$
is $\kappa_\beta\lambda$-strongly convex.
Since $\boldsymbol\theta^\star$ minimizes
this objective over the convex set $\Theta$, we have
\begin{equation}
\mathcal R_{\rm \method{}}
(\boldsymbol\theta)
-
\mathcal R_{\rm \method{}}
(\boldsymbol\theta^\star)
\ge
\frac{\kappa_\beta\lambda}{2}
\|\boldsymbol\theta-\boldsymbol\theta^\star\|_2^2.
\end{equation}

\HPOconvergence*

\textit{Proof. }
Now we are ready to prove Theorem~\ref{thm:hpo-main}.
By Lemma~\ref{lem:local-uniform-convergence}, with probability at least
\(1-\delta\), simultaneously for all \(g\in[m]\),
\begin{equation}
\sup_{\theta\in\Theta}
\left|
\mathcal R_g(\boldsymbol\theta)
-
\widehat{\mathcal R}_g(\boldsymbol\theta)
\right|
\le
\epsilon_g,
\end{equation}
where
\begin{equation}
\epsilon_g
=
2A_g\sqrt{\frac{d}{n_g}}
+
K_g
\sqrt{
\frac{\log(2m/\delta)}{2n_g}
}.
\end{equation}
By Lemma~\ref{lem:hierarchical-uniform-convergence}, we have
\begin{equation}
\sup_{\boldsymbol\theta\in\Theta}
\left|
\mathcal R_{\rm \method{}}(\boldsymbol\theta)
-
\widehat{\mathcal R}_{\rm \method{}}(\boldsymbol\theta)
\right|
\le
\epsilon_{\rm H},
\end{equation}
where
\begin{equation}
\epsilon_{\rm H}
=
\sup_{\boldsymbol\omega\in\mathcal U^\omega}
\sum_{g=1}^m
\omega_g\epsilon_g.
\end{equation}

Since \(\widehat{\boldsymbol\theta}\) minimizes the empirical \method{} objective, we have
\begin{equation}
\widehat{\mathcal R}_{\rm \method{}}(\widehat{\boldsymbol\theta})
\le
\widehat{\mathcal R}_{\rm \method{}}(\boldsymbol\theta^\star).
\end{equation}
Therefore,
\begin{equation}
\begin{aligned}
\mathcal R_{\rm \method{}}(\widehat{\boldsymbol\theta})
-
\mathcal R_{\rm \method{}}(\boldsymbol\theta^\star)
&=
\left[
\mathcal R_{\rm \method{}}(\widehat{\boldsymbol\theta})
-
\widehat{\mathcal R}_{\rm \method{}}(\widehat{\boldsymbol\theta})
\right]\\
&\quad+
\left[
\widehat{\mathcal R}_{\rm \method{}}(\widehat{\boldsymbol\theta})
-
\widehat{\mathcal R}_{\rm \method{}}(\boldsymbol\theta^\star)
\right]\\
&\quad+
\left[
\widehat{\mathcal R}_{\rm \method{}}(\boldsymbol\theta^\star)
-
\mathcal R_{\rm \method{}}(\boldsymbol\theta^\star)
\right]\\
&\le
\epsilon_{\rm H}
+
0
+
\epsilon_{\rm H}
=
2\epsilon_{\rm H}.
\end{aligned}
\end{equation}
By Lemma~\ref{lem:strong-convexity-coverage},
\begin{equation}
\mathcal R_{\rm \method{}}(\widehat{\boldsymbol\theta})
-
\mathcal R_{\rm \method{}}(\theta^\star)
\ge
\frac{\kappa_\beta\lambda}{2}
\|\widehat\theta-\theta^\star\|_2^2.
\end{equation}
Combining the two inequalities gives
\begin{equation}
\frac{\kappa_\beta\lambda}{2}
\|\widehat\theta-\theta^\star\|_2^2
\le
2\epsilon_{\rm H}.
\end{equation}
Thus,
\begin{equation}
\|\widehat\theta-\theta^\star\|_2
\le
2
\sqrt{
\frac{\epsilon_{\rm H}}{\kappa_\beta\lambda}
}.
\end{equation}
Finally, since \(\mathcal U^\omega\subseteq\Delta^{m-1}\), we have
\begin{equation}
\epsilon_{\rm H}
=
\sup_{\omega\in\mathcal U^\omega}
\sum_{g=1}^m
\omega_g\epsilon_g
\le
\max_{g\in[m]}\epsilon_g.
\end{equation}
Substituting the expression for \(\epsilon_g\) yields the stated max-form
bound.

\section{Experimental Details}
\label{app:experiment_details}

\subsection{Details for OOD preference alignment}
\label{app:emotion_details}

We construct a controlled benchmark from the \texttt{dair-ai/emotion}
dataset~\citep{saravia-etal-2018-carer}, which consists of short English social-media
texts labeled with six emotion categories. We remove the underrepresented
\texttt{surprise} class, retaining five emotions: \texttt{sadness}, \texttt{joy},
\texttt{love}, \texttt{anger}, and \texttt{fear}. The benchmark is built in three steps: we first train a multi-objective emotion reward model, then use it to construct preference pairs, and finally define latent groups and the OOD evaluation environments.

\paragraph{Multi-objective reward model.}
Following~\citet{xu2025neurips-robust_wdpo}, we convert the dataset into a multi-label reward dataset by randomly concatenating one to three source texts and merging their emotion labels into a five-dimensional binary vector $z\in\{0,1\}^5$. We fine-tune a Qwen2.5-0.5B-Instruct backbone with a linear classification head, trained with binary cross-entropy loss \[ \mathcal{L}_{\mathrm{RM}} = -\sum\nolimits_{k=1}^{5} \bigl[z_k\log\sigma(h_k(x))+(1-z_k)\log(1-\sigma(h_k(x)))\bigr], \] where $h(x)\in\mathbb{R}^5$ denotes the logits produced by the classification head applied to the hidden state of the last non-padding token. Training uses learning rate $5\times10^{-5}$, weight decay $0.01$, dropout $0.1$, and maximum sequence length $128$ for $8$ epochs.

\paragraph{Base model and preference pair construction.}
We fine-tune two base models: Qwen2.5-0.5B-Instruct and LLaMA-3.2-1B-Instruct via supervised fine-tuning (SFT) on the original emotion texts for $10$ epochs with learning rate $5\times10^{-6}$, weight decay $0.01$, max gradient norm $10$. The resulting models serve as both the initial policy and the reference model for preference optimization. For each prompt (first four tokens of each source text), we sample eight candidate completions with temperature $0.7$, top-$p=0.9$, max $64$ tokens, and score them with a scalar preference reward defined as a convex combination of two emotion rewards: $ R_{\alpha}(y) = \alpha\,r_{\texttt{anger}}(y) + (1-\alpha)\,r_{\texttt{fear}}(y). $ Training pairs are constructed under $\alpha_{\mathrm{train}}=0.2$ by selecting the highest- and lowest-scoring completions as chosen and rejected responses, respectively, and filtering out pairs with reward margin below $0.02$.

\paragraph{Group Construction and OOD Shift Settings.}
Since the Emotion dataset lacks explicit domain annotations, we construct six latent groups by applying $K$-means clustering ($K=6$) to the prompt and response embeddings extracted from Qwen2.5-0.5B-Instruct. We consider four OOD evaluation settings corresponding to the distribution shifts described in Section~\ref{sec:hierarchical-shifts}: \textbf{(i) Group-level shift:} We change the group mixture proportions from a skewed training distribution $(40\%,25\%,15\%,10\%,7\%,3\%)$ to a uniform test distribution. \textbf{(ii) Prompt-level shift:} We apply synonym substitution to at most one word per prompt from a predefined vocabulary (\texttt{feel}, \texttt{happy}, \texttt{sad}, \texttt{angry}, \texttt{afraid}, \texttt{love}). \textbf{(iii) Preference-level shift:} We change the reward mixing coefficient from $\alpha_{\mathrm{train}}=0.2$ to $\alpha_{\mathrm{test}}=0.8$, shifting the preference criterion from \texttt{fear} toward \texttt{anger}. \textbf{(iv) Combined shift:} We apply all three shifts simultaneously. Table~\ref{tab:splits} summarizes the evaluation settings.

\begin{table}[ht]
\centering
\caption{Environment construction for the Emotion preference benchmark.}
\label{tab:splits}
\begin{tabular}{lccc}
\toprule
Environment & Group Proportions & Prompt Perturbation & Reward Mixture $\alpha$ \\
\midrule
\texttt{train}           & Skewed  & None                 & $0.2$ \\
\texttt{val}             & Skewed  & None                 & $0.2$ \\
\texttt{id\_test}        & Skewed  & None                 & $0.2$ \\
\texttt{group\_test}    & Uniform & None                 & $0.2$ \\
\texttt{prompt\_test}    & Skewed  & Synonym substitution & $0.2$ \\
\texttt{preference\_test}& Skewed  & None                 & $0.8$ \\
\texttt{combined\_test}  & Uniform & Synonym substitution & $0.8$ \\
\bottomrule
\end{tabular}
\end{table}

\paragraph{Hyperparameter Search.}
We set the prompt adversarial perturbation step size $\eta_{\delta}=0.001$, the policy learning rate $\eta_{\theta}=1\times 10^{-6}$, global batch size $64$, inner adversarial steps $K=1$, and scaling coefficient $\beta=0.1$. All models are optimized with AdamW~\citep{loshchilov2019decoupled} and trained on 2 NVIDIA H100 80GB GPUs. We select the remaining hyperparameters of \method{} via grid search on the validation set over the ranges in Table~\ref{tab:hyperparam_search}, using average preference accuracy on the validation set as the selection criterion. For baselines, we use the same validation metric and tune over comparable candidate sets following their original papers when applicable. Otherwise, we adopt the reported defaults.

\begin{table}[ht]
\centering
\caption{Hyperparameter search ranges.}
\label{tab:hyperparam_search}
\renewcommand{\arraystretch}{1.05}
\begin{tabular}{ll}
\toprule
Hyperparameter & Search Range \\
\midrule
$\gamma_g$       & $\{0.1,\ 1.0\}$ \\
$\eta_{\omega}$  & $\{0.01,\ 0.03,\ 0.05,\ 0.1\}$ \\
$\rho^{\omega}$  & $\{0.2, 0.1, 0.01,\ 0.05\}$ \\
$\rho_g^X$       & $\{0.01,\ 0.001,\ 0.0005,\ 0.0001\}$ \\
$\rho_g^S$       & $\{0.2, 0.1,\ 0.01,\ 0.005,\ 0.001\}$ \\
\bottomrule
\end{tabular}
\end{table}

\subsection{Details for Preference Poisoning Attack} \label{app:poison_attack_details}

\paragraph{Benchmark construction.} For a selected poisoned example $(\boldsymbol x, \boldsymbol y^+, \boldsymbol y^-)$, poisoning replaces it with $(\boldsymbol x, \boldsymbol y^-, \boldsymbol y^+)$. We construct poisoned training sets from the controlled Emotion preference benchmark (Appendix~\ref{app:emotion_details}), swapping chosen and rejected responses for a selected subset of training pairs at poisoning rates $\{0.1,0.2,0.25,0.3,0.35,0.4\}$, while validation and test splits remain clean. We consider two attack strategies described in Section~\ref{sec:experiment}.

\paragraph{Training and evaluation setup.} All methods use learning rate $\eta_\theta=5\times10^{-6}$, batch size $64$, $\beta=0.1$, and AdamW~\citep{loshchilov2019decoupled} on 2 NVIDIA H100 80GB GPUs. To ensure a comparable tuning budget, we select method-specific robustness hyperparameters on the clean validation split. For rDPO, we tune $\epsilon\in\{0.01,0.1,0.2\}$. For Dr.DPO, we tune $\tau\in\{0.1,0.5,1.0\}$. For \method{}, we disable prompt-level and group-level robustness and tune the preference-level radius $\rho_g^S\in\{0.05,0.1,0.2\}$. All selected models are trained on poisoned preference  data and evaluated on the clean test split. For both the random and margin-based preference flip attacks, we report average preference accuracy to evaluate robustness to preference poisoning.

\subsection{Details for Downstream Generalization}
\label{app:ultrafeedback_details}

\paragraph{Dataset and preference construction.}
We use UltraFeedback~\citep{cui2023ultrafeedback} as the preference training source. It contains instruction-following prompts together with multiple model-generated responses and fine-grained quality annotations. Following \citet{meng2024simpo}, we first collect on-policy responses from the SFT model, and then construct candidate response pairs using the ArmoRM reward model~\citep{wang-etal-2024-interpretable} by selecting the highest- and lowest-scoring responses for each prompt. We use the naturally defined UltraFeedback categories as the group structure for robust preference alignment.

\paragraph{Training and evaluation setup.}
We fine-tune all models for one epoch using the AdamW optimizer~\citep{loshchilov2019decoupled} with learning rate $\eta_\theta=5\times10^{-6}$, global batch size $64$, and $\beta=0.1$, on 2 NVIDIA H100 80GB GPUs. For \method{}, the preference-level radius is selected from $\rho_g^S\in\{0.005, 0.01,0.05, 0.1, 0.2\}$. We evaluate on Open LLM Leaderboard v2~\citep{fourrier2024openllmleaderboardv2}, which comprises six benchmark families: IFEval, BBH, MATH-Hard, GPQA, MuSR, and MMLU-Pro.

\subsection{Additional Implementation Details}
\label{app:implementation_details}

Following embedding-space adversarial training~\citep{zhu2020freelb}, we apply the prompt perturbation $\boldsymbol{\delta}$ to the prompt-token embedding matrix, leaving response tokens unchanged. The same perturbation is shared across the chosen, rejected, and reference branches, with padding positions masked out. The reference policy is evaluated on the perturbed prompts with its parameters frozen. We use the quadratic penalty $\frac{\gamma_g}{2}\|\boldsymbol{\delta}\|_F^2$ and project perturbations onto $\|\boldsymbol{\delta}\|_F^2\le 2\rho_g^X$, where $\|\cdot\|_F$ denotes the Frobenius norm.

For group-weight updates, we adopt exponentiated gradient ascent~\citep{SagawaKHL20_groupdro,ramesh2024group_grpo}, followed by a KL projection onto the group-mixture ambiguity set. Given group-wise surrogate losses $\widetilde{\mathcal R}_g^{(t)}$, we compute provisional weights
$\widetilde{\omega}_g\propto
\omega_g^{(t-1)}
\exp(\eta_\omega\widetilde{\mathcal R}_g^{(t)})$
and update
\begin{equation}
\boldsymbol{\omega}^{(t)}
=
\arg\min_{\substack{\boldsymbol{\omega}\in\Delta^{m-1}\\
\|\boldsymbol{\omega}-\boldsymbol{\alpha}\|_1\le 2\rho^\omega}}
\mathcal D_{\mathrm{KL}}
(\boldsymbol{\omega}\,\|\,\widetilde{\boldsymbol{\omega}}).
\end{equation}
The $\ell_1$ constraint follows from
$\mathcal W_1^G(\boldsymbol{\omega},\boldsymbol{\alpha})
=\frac12\|\boldsymbol{\omega}-\boldsymbol{\alpha}\|_1$
under the discrete group transport cost.
We solve this convex projection using
\texttt{CVXPY}~\citep{diamond2016cvxpy}.

\begin{table*}[t]
\centering
\caption{Average and worst-group accuracy (\%) on the Emotion preference benchmark. }
\label{tab:emotion_avg_worst_std_appendix}

\begin{minipage}{\linewidth}
\centering
\vspace{0.1em}

\resizebox{\linewidth}{!}{%
\begin{tabular}{lcccccccccc}
\toprule
\multirow{2}{*}{Qwen-2.5-0.5B}
& \multicolumn{2}{c}{ID}
& \multicolumn{2}{c}{Group}
& \multicolumn{2}{c}{Prompt}
& \multicolumn{2}{c}{Preference}
& \multicolumn{2}{c}{Combined} \\
\cmidrule(lr){2-3}
\cmidrule(lr){4-5}
\cmidrule(lr){6-7}
\cmidrule(lr){8-9}
\cmidrule(lr){10-11}
& Avg Acc & Worst Acc
& Avg Acc & Worst Acc
& Avg Acc & Worst Acc
& Avg Acc & Worst Acc
& Avg Acc & Worst Acc \\
\midrule
DPO
& \sE{86.87{\tiny $\pm$0.32}} & \sA{77.78{\tiny $\pm$0.87}}
& \sA{81.80{\tiny $\pm$0.15}} & \sA{76.15{\tiny $\pm$0.92}}
& \sA{83.17{\tiny $\pm$0.46}} & \sA{76.98{\tiny $\pm$1.10}}
& \sB{82.93{\tiny $\pm$0.49}} & \sA{78.16{\tiny $\pm$1.99}}
& \sB{82.60{\tiny $\pm$0.92}} & \sC{76.40{\tiny $\pm$0.51}} \\

Dr.DPO
& \sA{84.81{\tiny $\pm$0.23}} & \sC{81.36{\tiny $\pm$1.02}}
& \sA{81.65{\tiny $\pm$0.35}} & \sB{77.06{\tiny $\pm$0.78}}
& \sA{82.75{\tiny $\pm$0.69}} & \sA{76.84{\tiny $\pm$1.22}}
& \sA{82.24{\tiny $\pm$0.44}} & \sA{75.86{\tiny $\pm$1.65}}
& \sA{81.27{\tiny $\pm$0.94}} & \sA{75.22{\tiny $\pm$0.73}} \\

WDPO
& \sA{84.87{\tiny $\pm$0.29}} & \sE{84.13{\tiny $\pm$1.18}}
& \sB{82.72{\tiny $\pm$0.16}} & \sD{78.90{\tiny $\pm$0.49}}
& \sE{85.08{\tiny $\pm$0.22}} & \sE{82.20{\tiny $\pm$2.36}}
& \sD{84.01{\tiny $\pm$0.65}} & \sD{81.46{\tiny $\pm$1.83}}
& \sD{84.21{\tiny $\pm$0.31}} & \sD{78.76{\tiny $\pm$1.14}} \\

\method{} (Group)
& \sB{84.90{\tiny $\pm$0.32}} & \sA{77.27{\tiny $\pm$1.52}}
& \sE{82.96{\tiny $\pm$0.23}} & \sA{76.76{\tiny $\pm$1.40}}
& \sB{83.46{\tiny $\pm$0.32}} & \sB{78.28{\tiny $\pm$0.87}}
& \sA{82.31{\tiny $\pm$0.22}} & \sB{79.24{\tiny $\pm$0.12}}
& \sA{82.06{\tiny $\pm$0.23}} & \sB{76.11{\tiny $\pm$0.88}} \\

\method{} (Prompt)
& \sD{85.93{\tiny $\pm$0.55}} & \sD{83.77{\tiny $\pm$0.76}}
& \sD{82.84{\tiny $\pm$0.26}} & \sE{80.43{\tiny $\pm$2.12}}
& \sD{84.69{\tiny $\pm$0.70}} & \sD{81.05{\tiny $\pm$2.11}}
& \sE{85.02{\tiny $\pm$0.97}} & \sF{83.03{\tiny $\pm$1.61}}
& \sE{84.37{\tiny $\pm$0.39}} & \sE{79.65{\tiny $\pm$0.88}} \\

\method{} (Preference)
& \sC{85.64{\tiny $\pm$0.22}} & \sB{78.28{\tiny $\pm$1.75}}
& \sC{82.77{\tiny $\pm$0.09}} & \sC{77.98{\tiny $\pm$0.92}}
& \sC{84.30{\tiny $\pm$0.43}} & \sC{79.75{\tiny $\pm$0.71}}
& \sC{83.28{\tiny $\pm$0.10}} & \sC{79.99{\tiny $\pm$0.74}}
& \sC{83.24{\tiny $\pm$0.34}} & \sB{76.11{\tiny $\pm$0.88}} \\

\method{}
& \sF{86.95{\tiny $\pm$0.70}} & \sF{85.12{\tiny $\pm$0.24}}
& \sF{84.61{\tiny $\pm$0.99}} & \sF{81.65{\tiny $\pm$0.62}}
& \sF{85.47{\tiny $\pm$0.86}} & \sF{82.56{\tiny $\pm$1.32}}
& \sF{85.44{\tiny $\pm$1.13}} & \sE{82.34{\tiny $\pm$2.84}}
& \sF{85.25{\tiny $\pm$0.78}} & \sF{79.94{\tiny $\pm$1.84}} \\
\bottomrule
\end{tabular}
}
\end{minipage}

\vspace{0.8em}

\begin{minipage}{\linewidth}
\centering
\vspace{-0.7em}

\resizebox{\linewidth}{!}{
\begin{tabular}{lcccccccccc}
\toprule
\multirow{2}{*}{LLaMA-3.2-1B}
& \multicolumn{2}{c}{ID}
& \multicolumn{2}{c}{Group}
& \multicolumn{2}{c}{Prompt}
& \multicolumn{2}{c}{Preference}
& \multicolumn{2}{c}{Combined} \\
\cmidrule(lr){2-3}
\cmidrule(lr){4-5}
\cmidrule(lr){6-7}
\cmidrule(lr){8-9}
\cmidrule(lr){10-11}
& Avg Acc & Worst Acc
& Avg Acc & Worst Acc
& Avg Acc & Worst Acc
& Avg Acc & Worst Acc
& Avg Acc & Worst Acc \\
\midrule
DPO
& \sB{88.75{\tiny $\pm$0.06}} & \sC{86.19{\tiny $\pm$0.15}}
& \sB{88.23{\tiny $\pm$0.38}} & \sE{83.44{\tiny $\pm$1.92}}
& \sB{87.81{\tiny $\pm$0.27}} & \sB{85.43{\tiny $\pm$0.83}}
& \sB{86.91{\tiny $\pm$0.22}} & \sF{81.42{\tiny $\pm$2.43}}
& \sB{87.44{\tiny $\pm$0.71}} & \sD{82.61{\tiny $\pm$0.87}} \\

Dr.DPO
& \sA{84.16{\tiny $\pm$0.72}} & \sA{81.00{\tiny $\pm$1.24}}
& \sA{84.65{\tiny $\pm$0.57}} & \sB{79.87{\tiny $\pm$2.29}}
& \sA{83.73{\tiny $\pm$0.62}} & \sA{80.98{\tiny $\pm$1.31}}
& \sA{82.87{\tiny $\pm$1.26}} & \sA{78.87{\tiny $\pm$2.21}}
& \sA{85.51{\tiny $\pm$0.66}} & \sA{79.13{\tiny $\pm$1.74}} \\

WDPO
& \sA{87.53{\tiny $\pm$1.03}} & \sB{84.17{\tiny $\pm$1.77}}
& \sA{87.24{\tiny $\pm$1.05}} & \sC{80.86{\tiny $\pm$0.57}}
& \sA{87.78{\tiny $\pm$1.03}} & \sC{85.82{\tiny $\pm$1.71}}
& \sA{85.28{\tiny $\pm$0.96}} & \sB{79.73{\tiny $\pm$1.82}}
& \sA{86.71{\tiny $\pm$1.31}} & \sB{81.16{\tiny $\pm$1.00}} \\

\method{} (Group)
& \sC{89.03{\tiny $\pm$0.39}} & \sE{87.02{\tiny $\pm$0.68}}
& \sD{88.61{\tiny $\pm$0.33}} & \sE{83.39{\tiny $\pm$2.12}}
& \sC{88.32{\tiny $\pm$0.33}} & \sD{85.86{\tiny $\pm$1.10}}
& \sC{86.94{\tiny $\pm$0.21}} & \sD{81.22{\tiny $\pm$2.31}}
& \sC{87.54{\tiny $\pm$0.38}} & \sE{82.90{\tiny $\pm$0.50}} \\

\method{} (Prompt)
& \sE{89.78{\tiny $\pm$1.32}} & \sD{86.37{\tiny $\pm$1.03}}
& \sC{88.50{\tiny $\pm$1.01}} & \sD{81.19{\tiny $\pm$0.99}}
& \sD{88.71{\tiny $\pm$0.93}} & \sC{85.41{\tiny $\pm$1.00}}
& \sD{87.30{\tiny $\pm$0.52}} & \sE{81.31{\tiny $\pm$1.75}}
& \sD{87.63{\tiny $\pm$0.17}} & \sC{82.03{\tiny $\pm$1.00}} \\

\method{} (Preference)
& \sD{89.75{\tiny $\pm$0.33}} & \sF{87.55{\tiny $\pm$0.43}}
& \sE{89.00{\tiny $\pm$0.74}} & \sE{83.17{\tiny $\pm$1.71}}
& \sE{89.00{\tiny $\pm$0.16}} & \sE{86.23{\tiny $\pm$0.06}}
& \sE{87.40{\tiny $\pm$0.54}} & \sE{81.24{\tiny $\pm$2.43}}
& \sE{87.87{\tiny $\pm$0.30}} & \sD{82.61{\tiny $\pm$0.00}} \\

\method{}
& \sF{90.57{\tiny $\pm$0.27}} & \sF{87.55{\tiny $\pm$0.43}}
& \sF{89.11{\tiny $\pm$0.50}} & \sF{83.50{\tiny $\pm$1.14}}
& \sF{89.14{\tiny $\pm$0.47}} & \sF{86.70{\tiny $\pm$0.74}}
& \sF{88.25{\tiny $\pm$0.34}} & \sE{81.31{\tiny $\pm$0.87}}
& \sF{88.31{\tiny $\pm$0.30}} & \sF{83.19{\tiny $\pm$0.50}} \\
\bottomrule
\end{tabular}
}
\end{minipage}

\end{table*}

\section{Additional Experiments}
\label{app:additional_experiments}

\paragraph{Detailed results of Emotion benchmark.}
The main paper reports mean results. Here we additionally report the standard deviations across three runs with different random seeds. The results are shown in Table~\ref{tab:emotion_avg_worst_std_appendix}.

\paragraph{Hyperparameter sensitivity analysis.}

We further study the sensitivity of \method{} to the three robustness radii on the OOD preference alignment benchmark using the LLaMA-3.2-1B-Instruct model. Figures~\ref{fig:emotion_sensitivity_target} and~\ref{fig:emotion_sensitivity_combined} show the results obtained by varying one radius at a time while keeping the other robustness components fixed. The results suggest that \method{} is generally stable across a broad range of values. The group radius \(\rho^\omega\) has a relatively mild effect, while performance improves as the prompt radius \(\rho_g^X\) increases from very small values and then quickly saturates. The preference radius \(\rho_g^S\) is more influential: moderate values improve robustness, but overly large values can reduce performance, indicating possible over-regularization. These trends offer practical guidance for robustness-radius selection and motivate future work on task-adaptive tuning strategies.

\begin{figure*}[t]
    \centering
    \begin{subfigure}[t]{0.32\textwidth}
        \centering
        \includegraphics[width=\linewidth]{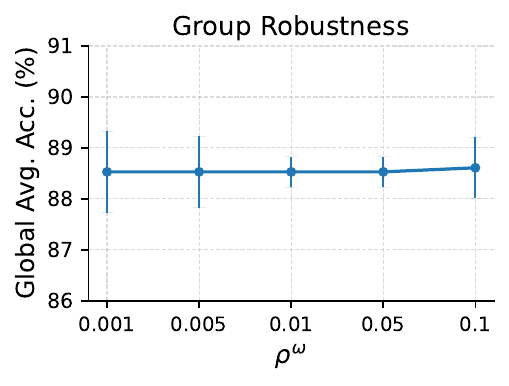}
        \caption{Group-level robustness.}
        \label{fig:emotion_sensitivity_global_target}
    \end{subfigure}
    \hfill
    \begin{subfigure}[t]{0.32\textwidth}
        \centering
        \includegraphics[width=\linewidth]{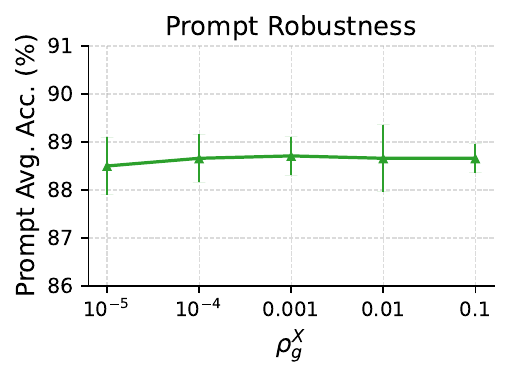}
        \caption{Prompt-level robustness.}
        \label{fig:emotion_sensitivity_prompt_target}
    \end{subfigure}
    \hfill
    \begin{subfigure}[t]{0.32\textwidth}
        \centering
        \includegraphics[width=\linewidth]{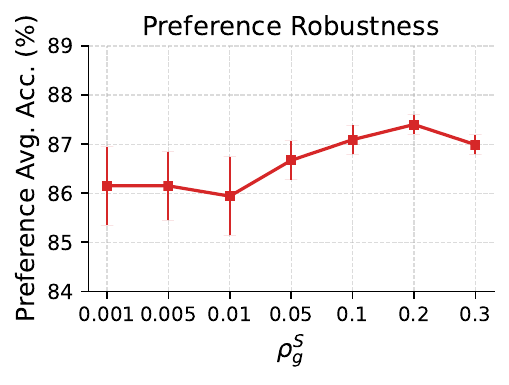}
        \caption{Preference-level robustness.}
        \label{fig:emotion_sensitivity_preference_target}
    \end{subfigure}
    \caption{
    Sensitivity of group-, prompt-, and preference-level robustness to their corresponding radii. 
    }
    \label{fig:emotion_sensitivity_target}
\end{figure*}

\begin{figure*}[t]
    \centering
    \begin{subfigure}[t]{0.32\textwidth}
        \centering
        \includegraphics[width=\linewidth]{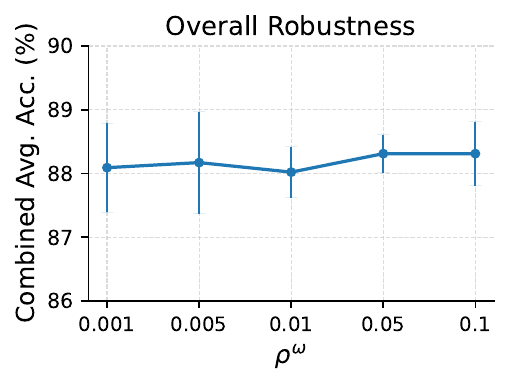}
        \caption{Group radius.}
        \label{fig:emotion_sensitivity_global_combined}
    \end{subfigure}
    \hfill
    \begin{subfigure}[t]{0.32\textwidth}
        \centering
        \includegraphics[width=\linewidth]{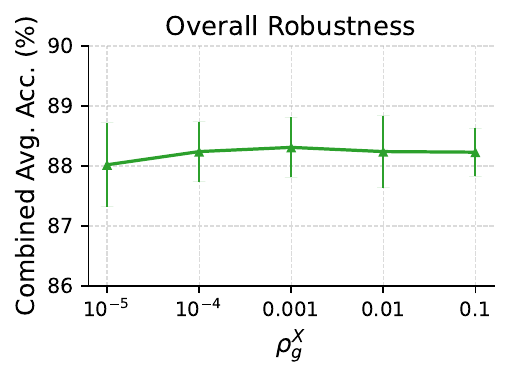}
        \caption{Prompt radius.}
        \label{fig:emotion_sensitivity_prompt_combined}
    \end{subfigure}
    \hfill
    \begin{subfigure}[t]{0.32\textwidth}
        \centering
        \includegraphics[width=\linewidth]{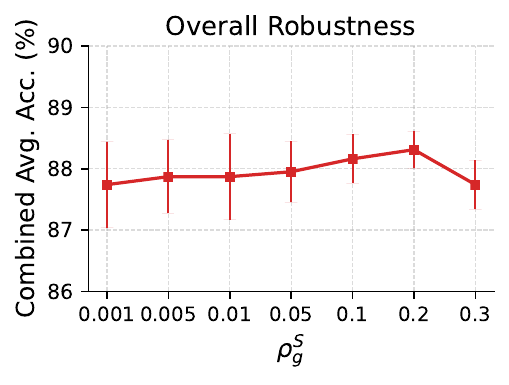}
        \caption{Preference radius.}
        \label{fig:emotion_sensitivity_preference_combined}
    \end{subfigure}
    \caption{
    Sensitivity of overall robustness to the three robustness radii.
    }
    \label{fig:emotion_sensitivity_combined}
\end{figure*}

\paragraph{Effect of Model Family and Scale.}
\label{app:model_family_scale}
To study the effect of model family and scale, we further evaluate \method{} on Qwen2.5-3B and LLaMA-3.2-3B. Together with Qwen2.5-0.5B and LLaMA-3.2-1B, this covers two model families, each evaluated at two model sizes. Table~\ref{tab:model_family_scale} reports the detailed results across ID and the four OOD environments, while Table~\ref{tab:model_family_scale_gain} summarizes performance averaged over Group, Prompt, Preference, and Combined. As shown in Table~\ref{tab:model_family_scale_gain}, \method{} improves both average and worst-group accuracy for all four models, with the largest gains observed on LLaMA-3.2-3B. Within each model family, lower DPO baseline accuracy is associated with larger gains in average accuracy, suggesting a headroom effect. In contrast, worst-group improvements show no consistent trend with model scale. Overall, these results suggest that \method{} remains effective across model families, scales, and OOD shift settings.

\begin{table*}[t]
\centering
\caption{
Preference accuracy (\%) across model families and scales.
}
\label{tab:model_family_scale}
\resizebox{\textwidth}{!}{%
\begin{tabular}{llcccccccccc}
\toprule
\multirow{2}{*}{Model}
& \multirow{2}{*}{Method}
& \multicolumn{2}{c}{ID}
& \multicolumn{2}{c}{Group}
& \multicolumn{2}{c}{Prompt}
& \multicolumn{2}{c}{Preference}
& \multicolumn{2}{c}{Combined} \\
\cmidrule(lr){3-4}
\cmidrule(lr){5-6}
\cmidrule(lr){7-8}
\cmidrule(lr){9-10}
\cmidrule(lr){11-12}
& & Avg & Worst & Avg & Worst & Avg & Worst
& Avg & Worst & Avg & Worst \\
\midrule

\multirow{2}{*}{Qwen2.5-0.5B}
& DPO
& 86.87 & 77.78
& 81.80 & 76.15
& 83.17 & 76.98
& 82.93 & 78.16
& 82.60 & 76.40 \\
& \method{}
& \textbf{86.95} & \textbf{85.12}
& \textbf{84.61} & \textbf{81.65}
& \textbf{85.47} & \textbf{82.56}
& \textbf{85.44} & \textbf{82.34}
& \textbf{85.25} & \textbf{79.94} \\
\midrule

\multirow{2}{*}{LLaMA-3.2-1B}
& DPO
& 88.75 & 86.19
& 88.23 & 83.44
& 87.81 & 85.43
& 86.91 & \textbf{81.42}
& 87.44 & 82.61 \\
& \method{}
& \textbf{90.57} & \textbf{87.55}
& \textbf{89.11} & \textbf{83.50}
& \textbf{89.14} & \textbf{86.70}
& \textbf{88.25} & 81.31
& \textbf{88.31} & \textbf{83.19} \\
\midrule

\multirow{2}{*}{Qwen2.5-3B}
& DPO
& 78.60 & 77.48
& 76.04 & 74.53
& 77.38 & 76.92
& 73.97 & 72.62
& 76.92 & 71.99 \\
& \method{}
& \textbf{79.69} & \textbf{78.12}
& \textbf{79.56} & \textbf{78.47}
& \textbf{80.34} & \textbf{77.61}
& \textbf{80.14} & \textbf{79.75}
& \textbf{80.08} & \textbf{73.26} \\
\midrule

\multirow{2}{*}{LLaMA-3.2-3B}
& DPO
& 85.24 & 84.04
& 83.54 & 80.69
& 86.20 & 84.68
& 82.74 & 80.00
& 83.43 & 80.83 \\
& \method{}
& \textbf{90.55} & \textbf{88.56}
& \textbf{89.82} & \textbf{86.92}
& \textbf{91.05} & \textbf{89.12}
& \textbf{87.36} & \textbf{84.31}
& \textbf{88.66} & \textbf{84.72} \\
\bottomrule
\end{tabular}%
}
\end{table*}

\begin{table}[t]
\centering
\small
\caption{
Average OOD performance and absolute gain of \method{} over DPO.
}
\label{tab:model_family_scale_gain}
\begin{tabular}{lcccccc}
\toprule
\multirow{2}{*}{Model}
& \multicolumn{3}{c}{OOD Average}
& \multicolumn{3}{c}{OOD Worst} \\
\cmidrule(lr){2-4}
\cmidrule(lr){5-7}
& DPO & \method{} & Gain
& DPO & \method{} & Gain \\
\midrule
Qwen2.5-0.5B
& 82.63 & 85.19 & +2.57
& 76.92 & 81.62 & +4.70 \\
Qwen2.5-3B
& 76.08 & 80.03 & +3.95
& 74.02 & 77.27 & +3.26 \\
LLaMA-3.2-1B
& 87.60 & 88.70 & +1.11
& 83.23 & 83.68 & +0.45 \\
LLaMA-3.2-3B
& 83.98 & 89.22 & +5.24
& 81.55 & 86.27 & +4.72 \\
\bottomrule
\end{tabular}%
\end{table}

\paragraph{Downstream Evaluation with Robust DPO Baselines.}
\label{app:downstream_robust_baselines}

We further evaluate the downstream generalization of robust preference optimization methods on the Open LLM Leaderboard v2~\citep{fourrier2024openllmleaderboardv2}. In addition to standard DPO and \method{}, we include two representative robust DPO baselines: WDPO~\citep{xu2025neurips-robust_wdpo} and Dr.DPO~\citep{wu2025iclr-robust_drdpo}. All methods use LLaMA-3.2-1B-Instruct with the same UltraFeedback preference data, SFT initialization, training schedule, and evaluation protocol. We evaluate the resulting models on MMLU-Pro, BBH, GPQA, MATH-Hard, IFEval, and MuSR. As shown in Table~\ref{tab:downstream_robust_baselines}, the methods exhibit different strengths across downstream tasks. Dr.DPO achieves the best performance on BBH, GPQA, MATH-Hard, and MuSR, while \method{} performs best on MMLU-Pro and IFEval. Overall, the performance differences are relatively small across most benchmarks, with \method{} remaining competitive with both standard DPO and the robust DPO baselines. These results suggest that the robustness benefits of \method{} are compatible with strong general downstream performance.

\begin{table}[t]
\centering
\small
\caption{
Downstream evaluation on the Open LLM Leaderboard v2 using
LLaMA-3.2-1B-Instruct. The best result is shown in bold and the second-best is underlined.
}
\label{tab:downstream_robust_baselines}
\begin{tabular}{lccccccc}
\toprule
Method
& MMLU-Pro
& BBH
& GPQA
& MATH-Hard
& IFEval
& MuSR
& Avg. \\
\midrule

DPO
& \underline{20.0}
& 35.6
& 23.2
& 8.1
& 51.3
& 31.6
& 28.30 \\

WDPO
& 19.4
& \underline{36.1}
& \underline{24.1}
& 8.0
& 47.9
& \underline{32.5}
& 28.00 \\

Dr.DPO
& 19.7
& \textbf{36.4}
& \textbf{24.3}
& \textbf{8.6}
& \underline{52.1}
& \textbf{32.6}
& \textbf{28.95} \\

\method{} ($\rho_g^S=0.01$)
& \textbf{20.5}
& 35.9
& 23.7
& \underline{8.3}
& \textbf{53.0}
& 31.2
& \underline{28.77} \\

\bottomrule
\end{tabular}%
\end{table}

\paragraph{Computational Cost.} 
\begin{wrapfigure}{r}{0.45\linewidth}

\centering
\vspace{-1.2em}
\includegraphics[width=1.0\linewidth]{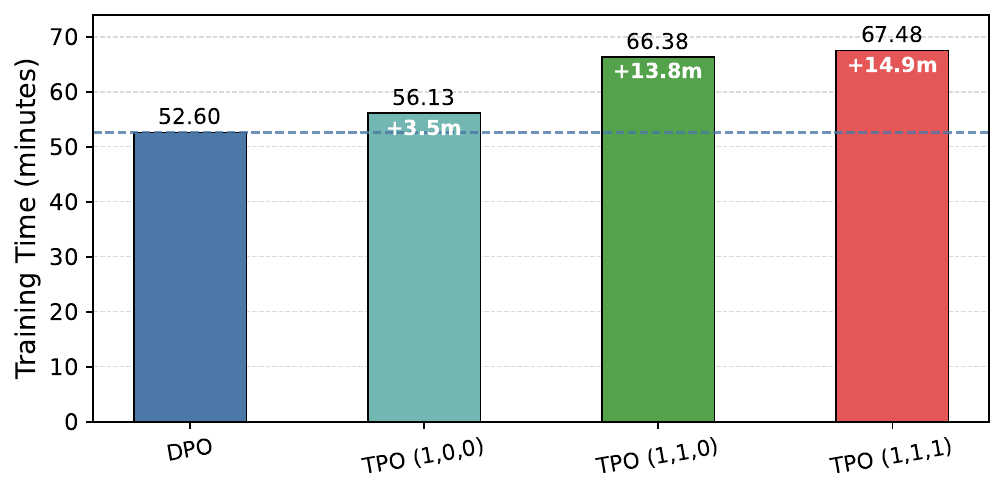}
\caption{Computational cost on UltraFeedback with LLaMA-3.2-1B-Instruct.}
\vspace{-0.5em}
\label{fig:computational_cost}
\end{wrapfigure}
We report the wall-clock training time of LLaMA-3.2-1B-Instruct on the UltraFeedback setting in Figure~\ref{fig:computational_cost} using 2 NVIDIA H100 80GB GPUs. The time is measured in minutes. Here, \method{}\((1,0,0)\), \method{}\((1,1,0)\), and \method{}\((1,1,1)\) denote enabling group-level robustness only, enabling both group- and prompt-level robustness, and enabling all three robustness components, respectively. The corresponding overheads are approximately \(+3.5\), \(+13.8\), and \(+14.9\) minutes compared with DPO. These results show that \method{} introduces additional computation compared with standard DPO, mainly due to group reweighting and the inner adversarial prompt perturbation steps. Nevertheless, the overall overhead remains modest and manageable in our implementation.

\section{Limitations}
\label{sec:limitations}

First, our finite-sample convergence guarantee relies on the uniform covariance coverage assumption (Assumption~\ref{ass:uniform-coverage}), which requires sufficient coverage of pairwise feature differences uniformly over admissible group mixtures and worst-case prompt perturbations. This assumption is standard in preference learning theory, though relaxing it to broader settings remains an interesting theoretical direction. 

Second, our prompt-level robustification uses embedding-space perturbations as a practical approximation to shift in the discrete token space, which is a widely adopted strategy in robust NLP but may not capture all forms of prompt-level variation in real-world scenarios, though it has been shown to be effective in our controlled OOD alignment experiments. Further work may explore more robust prompt-level perturbations in the discrete token space, such as adversarial token substitutions or paraphrasing.

Third, \method{} introduces additional hyperparameters, including the robustness radii \(\{\rho^{\omega}, \rho_g^X, \rho_g^S\}\). Since the appropriate ranges of these radii may vary across tasks and scenarios, selecting them can require additional effort. Although our sensitivity analysis provides some practical guidance, developing principled selection rules or task-adaptive tuning strategies remains an important direction for future work.

\section{Impact Statement}
\label{sec:impact}
This work proposes a hierarchical distributionally robust framework for preference optimization, with the primary goal of improving the reliability of LLM alignment across diverse user populations and deployment environments, which may help reduce performance disparities across underrepresented groups. As a foundational algorithmic contribution, it does not involve model deployment. More robustly aligned models could nonetheless be misused to generate more persuasive harmful content across a broader range of contexts, a risk shared with LLM alignment research more broadly.


\newpage
\clearpage
\section*{NeurIPS Paper Checklist}

\begin{enumerate}

\item {\bf Claims}
    \item[] Question: Do the main claims made in the abstract and introduction accurately reflect the paper's contributions and scope?
    \item[] Answer: \answerYes{} 
    \item[] Justification: Our paper's contributions and scope have been clearly described in the abstract and introduction.

\item {\bf Limitations}
    \item[] Question: Does the paper discuss the limitations of the work performed by the authors?
    \item[] Answer: \answerYes{} 
    \item[] Justification: Please refer to Appendix~\ref{sec:limitations}.

\item {\bf Theory assumptions and proofs}
    \item[] Question: For each theoretical result, does the paper provide the full set of assumptions and a complete (and correct) proof?
    \item[] Answer: \answerYes{} 
    \item[] Justification: We have clearly stated all assumptions used in our work.

    \item {\bf Experimental result reproducibility}
    \item[] Question: Does the paper fully disclose all the information needed to reproduce the main experimental results of the paper to the extent that it affects the main claims and/or conclusions of the paper (regardless of whether the code and data are provided or not)?
    \item[] Answer: \answerYes{} 
    \item[] Justification: We provide dataset construction, model training, hyperparameter search, and evaluation details in Appendices~\ref{app:emotion_details}, \ref{app:poison_attack_details}, and \ref{app:ultrafeedback_details}.

\item {\bf Open access to data and code}
    \item[] Question: Does the paper provide open access to the data and code, with sufficient instructions to faithfully reproduce the main experimental results, as described in supplemental material?
    \item[] Answer: \answerNo{} 
    \item[] Justification: While the data used in our study is publicly available, we are currently unable to provide open access to the code. However, we have included detailed descriptions of the data access, preprocessing steps, model architecture, and experimental setup in the paper and Appendix. These details should be sufficient for others to reproduce the experimental
results.

\item {\bf Experimental setting/details}
    \item[] Question: Does the paper specify all the training and test details (e.g., data splits, hyperparameters, how they were chosen, type of optimizer) necessary to understand the results?
    \item[] Answer: \answerYes{} 
    \item[] Justification: Please refer to Appendix~\ref{app:emotion_details}, \ref{app:poison_attack_details}, and \ref{app:ultrafeedback_details}.

\item {\bf Experiment statistical significance}
    \item[] Question: Does the paper report error bars suitably and correctly defined or other appropriate information about the statistical significance of the experiments?
    \item[] Answer: \answerYes{} 
    \item[] Justification: We report mean and standard deviation over three random seeds for the main Emotion benchmark in Table~\ref{tab:emotion_avg_worst_std_appendix}.

\item {\bf Experiments compute resources}
    \item[] Question: For each experiment, does the paper provide sufficient information on the computer resources (type of compute workers, memory, time of execution) needed to reproduce the experiments?
    \item[] Answer: \answerYes{} 
    \item[] Justification: We provide the details in Appendix~\ref{app:experiment_details}. 
    
\item {\bf Code of ethics}
    \item[] Question: Does the research conducted in the paper conform, in every respect, with the NeurIPS Code of Ethics \url{https://neurips.cc/public/EthicsGuidelines}?
    \item[] Answer: \answerYes{} 
    \item[] Justification: We acknowledge the NeurIPS Code of Ethics.

\item {\bf Broader impacts}
    \item[] Question: Does the paper discuss both potential positive societal impacts and negative societal impacts of the work performed?
    \item[] Answer: \answerYes{} 
    \item[] Justification: See Appendix~\ref{sec:impact}.
    
\item {\bf Safeguards}
    \item[] Question: Does the paper describe safeguards that have been put in place for responsible release of data or models that have a high risk for misuse (e.g., pre-trained language models, image generators, or scraped datasets)?
    \item[] Answer: \answerNA{} 
    \item[] Justification: This paper has no such risks.

\item {\bf Licenses for existing assets}
    \item[] Question: Are the creators or original owners of assets (e.g., code, data, models), used in the paper, properly credited and are the license and terms of use explicitly mentioned and properly respected?
    \item[] Answer: \answerYes{} 
    \item[] Justification: We state the assumptions explicitly and provide complete proofs of the theoretical results in the appendix.

\item {\bf New assets}
    \item[] Question: Are new assets introduced in the paper well documented and is the documentation provided alongside the assets?
    \item[] Answer: \answerNA{} 
    \item[] Justification: The paper does not release new assets.

\item {\bf Crowdsourcing and research with human subjects}
    \item[] Question: For crowdsourcing experiments and research with human subjects, does the paper include the full text of instructions given to participants and screenshots, if applicable, as well as details about compensation (if any)? 
    \item[] Answer: \answerNA{} 
    \item[] Justification: [NA]

\item {\bf Institutional review board (IRB) approvals or equivalent for research with human subjects}
    \item[] Question: Does the paper describe potential risks incurred by study participants, whether such risks were disclosed to the subjects, and whether Institutional Review Board (IRB) approvals (or an equivalent approval/review based on the requirements of your country or institution) were obtained?
    \item[] Answer: \answerNA{} 
    \item[] Justification: [NA]

\item {\bf Declaration of LLM usage}
    \item[] Question: Does the paper describe the usage of LLMs if it is an important, original, or non-standard component of the core methods in this research? Note that if the LLM is used only for writing, editing, or formatting purposes and does \emph{not} impact the core methodology, scientific rigor, or originality of the research, declaration is not required.
    \item[] Answer: \answerYes{} 
    \item[] Justification: LLMs are the target models in our empirical evaluation, and the paper describes the model backbones, fine-tuning setup, and evaluation protocols.

\end{enumerate}

\end{document}